\documentclass[11pt]{article}
\usepackage{enumitem}
\usepackage[dvipsnames]{xcolor}
\usepackage{customcommands}

\usepackage[style = alphabetic,citestyle=alphabetic,maxbibnames=99,sorting=nyt,natbib=true,backref=true]{biblatex}
\DefineBibliographyStrings{english}{%
  backrefpage = {Cited on page},% originally "cited on page"
  backrefpages = {Cited on pages},% originally "cited on pages"
}
\usepackage[colorlinks,linkcolor = RawSienna, urlcolor  = Green, citecolor = Green, anchorcolor = ForestGreen,bookmarks=False]{hyperref}

\renewcommand{\r}{\right}
\renewcommand{\l}{\left} 

\author{
Ruiqi Zhang \\
UC Berkeley\\
\texttt{rqzhang@berkeley.edu}
\and Spencer Frei\\
 UC Berkeley\\
 \texttt{frei@berkeley.edu}\\
 \and
 Peter L. Bartlett\\
 UC Berkeley and Google DeepMind\\
 \texttt{peter@berkeley.edu}
}

\title{\textbf{ Trained Transformers Learn Linear Models In-Context}}   

\begin{document}

\maketitle

\begin{abstract}
Attention-based neural networks such as transformers have demonstrated a remarkable ability to exhibit in-context learning (ICL): Given a short prompt sequence of tokens from an unseen task, they can formulate relevant per-token and next-token predictions without any parameter updates.  By embedding a sequence of labeled training data and unlabeled test data as a prompt, this allows for transformers to behave like supervised learning algorithms. Indeed, recent work has shown that when training transformer architectures over random instances of linear regression problems, these models' predictions mimic those of ordinary least squares. 

Towards understanding the mechanisms underlying this phenomenon, we investigate the dynamics of ICL in transformers with a single linear self-attention layer trained by gradient flow on linear regression tasks.  We show that despite non-convexity, gradient flow with a suitable random initialization finds a global minimum of the objective function.  At this global minimum, when given a test prompt of labeled examples from a new prediction task, the transformer achieves prediction error competitive with the best linear predictor over the test prompt distribution. We additionally characterize the robustness of the trained transformer to a variety of distribution shifts and show that although a number of shifts are tolerated, shifts in the covariate distribution of the prompts are not.  Motivated by this, we consider a generalized ICL setting where the covariate distributions can vary across prompts. We show that although gradient flow succeeds at finding a global minimum in this setting, the trained transformer is still brittle under mild covariate shifts.  
We complement this finding with experiments on large, nonlinear transformer architectures which we show are more robust under covariate shifts.

\end{abstract}

\newcommand{\sfnote}[1]{\footnote{ #1 -sf.}}
\newcommand{\rznote}[1]{\footnote{ #1 -rz.}}
\newcommand{\pbnote}[1]{\footnote{ #1 -pb.}}

\def\softmax{\mathrm{softmax}}
\def\P{\mathbb{P}}
\def\vecUWX{\vector\left(U_\tau U X_\tau\right)}
\def\UUX{\left(U_\tau U X_\tau\right)}
\def\bU{U}
\def\bUtl{U_{11}}
\def\bUtr{u_{12}}
\def\bUbl{u_{21}}
\def\bUlast{u_{-1}}
\def\vecZUX{\vector\left(Z_\tau U X_\tau\right)}
\def\ZUX{\left(Z_\tau U X_\tau\right)}

\section{Introduction}

Transformer-based neural networks have quickly become the default machine learning model for problems in natural language processing, forming the basis of chatbots like ChatGPT~\citep{openai2023gpt4}, and are increasingly popular in computer vision~\citep{dosovitskiy2021visiontransformer}.  These models can take as input sequences of tokens and return relevant next-token predictions.  When trained on sufficiently large and diverse datasets, these models are often able to perform \textit{in-context learning} (ICL): when given a short sequence of input-output pairs (called a \textit{prompt}) from a particular task as input, the model can formulate predictions on test examples without having to make any updates to the parameters in the model.  

Recently,~\citet{garg2022can} initiated the investigation of ICL from the perspective of learning particular function classes.  At a high-level, this refers to when the model has access to instances of prompts of the form $(x_1, h(x_1), \dots, x_N, h(x_N), x_\query)$ where $x_i, x_\query$ are sampled i.i.d. from a distribution $\calD_x$ and $h$ is sampled independently from a distribution over functions in a function class $\calH$.  The transformer succeeds at in-context learning if when given a new prompt $(x_1', h'(x_1'),\dots, x_N', h'(x_N'), x_\query')$ corresponding to an independently sampled $h'$ it is able to formulate a prediction for $x_\query'$ that is close to $h'(x_\query')$ given a sufficiently large number of examples $N$.  The authors showed that when transformer models
are trained on prompts corresponding to instances of training data from a particular function class (e.g., linear models, neural networks, or decision trees), they succeed at in-context learning, and moreover the behavior of the trained transformers can mimic those of familiar learning algorithms like ordinary least squares.

Following this, a number of follow-up works provided constructions of transformer-based neural network architectures which are capable of achieving small prediction error for query examples when the prompt takes the form $(x_1, \sip{w}{x_1}, \dots, x_N, \sip{w}{x_N}, x_\query)$ where $x_i, x_\query, w\iid \normal(0, I_d)$~\citep{von2022transformers,akyurek2022learning}.  However, this leaves open the question of how it is that \textit{gradient-based optimization algorithms} over transformer architectures produce models which are capable of in-context learning.\footnote{We note a concurrent work also explores the optimization question we consider here~\citep{ahn2023transformers}; we shall provide a more detailed comparison to this work in Section~\ref{sec:related}.} 

In this work, we investigate the learning dynamics of gradient flow in a simplified transformer architecture when the training prompts consists of random instances of linear regression datasets.  Our main contributions are as follows. 
\begin{itemize} 
    \item We establish that for a class of transformers with a single layer and with a linear self-attention module (LSAs), gradient flow on the population loss with a suitable random initialization converges to a global minimum of the population objective, despite the non-convexity of the underlying objective function.
    \item We characterize the learning algorithm that is encoded by the transformer at convergence, as well as the prediction error achieved when the model is given a test prompt corresponding to a new (and possibly nonlinear) prediction task. 
    \item We use this to conclude that transformers trained by gradient flow indeed in-context learn the class of linear models.  Moreover, we characterize the robustness of the trained transformer to a variety of distribution shifts.  We show that although a number of shifts can be tolerated, shifts in the covariate distribution of the features $x_i$ can not.  
    \item Motivated by this failure under covariate shift, we consider a generalized setting of in-context learning where the covariate distribution can vary across prompts.  We provide global convergence guarantees for LSAs trained by gradient flow in this setting and show that even when trained on a variety of covariate distributions, LSAs still fail under covariate shift. \item We then empirically investigate the behavior of large, nonlinear transformers when trained on linear regression prompts.  We find that these more complex models are able to generalize better under covariate shift, especially when trained on prompts with varying covariate distributions.
\end{itemize}

\section{Additional Related Work}\label{sec:related}
The literature on transformers and non-convex optimization in machine learning is vast.  In this section, we will focus on those works most closely related to theoretical understanding of in-context learning of function classes. 

As mentioned previously,~\citet{garg2022can} empirically investigated the ability for transformer architectures to in-context learn a variety of function classes.  They showed that when trained on random instances of linear regression, the models' predictions are very similar to those of ordinary least squares.  Additionally, they showed that transformers can in-context learn two-layer ReLU networks and decision trees, showing that by training on differently-structured data, the transformers learn to implement distinct learning algorithms.   A number of works further investigated the types of algorithms implemented by transformers trained on in-context examples of linear models~\citep{ahuja2023context,ahuja2023closer}.

\citet{akyurek2022learning} and \citet{von2022transformers} examined the behavior of transformers when trained on random instances of linear regression, as we do in this work.  They considered the setting of isotropic Gaussian data with isotropic Gaussian weight vectors, and showed that the trained transformer's predictions mimic those of a single step of gradient descent.  They also provided a construction of transformers which implement this single step of gradient descent.  By contrast, we explicitly show that gradient flow provably converges to transformers which learn linear models in-context.  Moreover, our analysis holds when the covariates are anisotropic Gaussians, for which a single step of vanilla gradient descent is unable to achieve small prediction error.\footnote{To see this, suppose $(x_i, y_i)$ are i.i.d. with $x \sim \normal(0,\Lambda)$ and $y=\sip wx$.  A single step of gradient descent under the squared loss from a zero initialization yields the predictor $x\mapsto x^\top\l(\f 1 n \summ i n y_i x_i \r)= x^\top \l( \f 1 n \summ i n x_i x_i^\top \r) w \approx x^\top \Lambda w$.  Clearly, this is not close to $x^\top w$ when $\Lambda \neq I_d$.} 

Let us briefly mention a number of other works on understanding in-context learning in transformers and other sequence-based models.
 \citet{han2023incontext} suggests that Bayesian inference on prompts can be asymptotically interpreted as kernel regression. \citet{dai2022can} interprets ICL as implicit fine-tuning, viewing large language models as meta-optimizers performing gradient-based optimization. \citet{xie2021explanation} regards ICL as implicit Bayesian inference, with transformers learning a shared latent concept between prompts and test data, and they prove the ICL property when the training distribution is a mixture of HMMs. Similarly, \citet{wang2023large} perceives ICL as a Bayesian selection process, implicitly inferring information pertinent to the designated tasks. \citet{li2023closeness} explores the functional resemblance between a single layer of self-attention and gradient descent on a softmax regression problem, offering upper bounds on their difference. \citet{min2022rethinking} notes that the alteration of label parts in prompts does not drastically impair the ICL ability. They contend that ICL is invoked when prompts reveal information about the label space, input distribution, and sequence structure. 

Another collection of works have sought to understand transformers from an approximation-theoretic perspective. \citet{yun2019transformers, yun2020n} established that transformers can universally approximate any sequence-to-sequence function under some assumptions. Investigations by \citet{edelman2022inductive, likhosherstov2021expressive} indicate that a single-layer self-attention can learn sparse functions of the input sequence, where sample complexity and hidden size are only logarithmic relative to the sequence length. Further studies by \citet{perez2019turing, dehghani2019universal, bhattamishra2020computational} indicate that the vanilla transformer and its variants exhibit Turing completeness.  \citet{liu2023transformers} showed that transformers can approximate finite-state automata with few layers. \citet{bai2023transformers} showed that transformers can implement a variety of statistical machine learning algorithms as well as model selection procedures.
\citet{aamw-23} showed that a pretrained transformer can be used to define a transformer that segments a prompt into examples and labels and learns to solve a sparse retrieval task. \citet{zhang2023and} interpreted in-context learning via a Bayesian model averaging process.

A handful of recent works have developed provable guarantees for transformers trained with gradient-based optimization.  \citet{jelassi2022vision} analyzed the dynamics of gradient descent in vision transformers for data with spatial structure.  \citet{li2023transformerstopic} demonstrated that a single-layer transformer trained by a gradient method could learn a topic model, treating learning semantic structure as detecting co-occurrence between words and theoretically analyzing the two-stage dynamics during the training process.

Finally, we note a concurrent work by~\citet{ahn2023transformers} on the optimization landscape of single layer transformers with linear self-attention layers as we do in this work.  
They show that there exist global minima of the population objective of the transformer that can achieve small prediction error with anisotropic Gaussian data, and they characterize some critical points of deep linear self-attention networks.  In this work, we show that despite nonconvexity, gradient flow with a suitable random initialization converges to a global minimum that achieves small prediction error for anistropic Gaussian data.  We also characterize the prediction error when test prompts come from a new (possibly nonlinear) task, when there is distribution shift, and when transformers are trained on prompts with possibly different covariate distributions across prompts.  

\section{Preliminaries}
\paragraph{Notation}
We first describe the notation we use in the paper.  
We write $[n] = \{1,2,...,n\}.$  We use $\otimes$ to denote the Kronecker product, and $\vector$ the vectorization operator in column-wise order. For example, $\vector \begin{pmatrix}[0.5] 1 & 2 \\ 3 & 4\end{pmatrix} = (1,3,2,4)^\top.$ We write the inner product of two matrices $A,B \in \mathbb{R}^{m \times n}$ as $\left\langle A,B\right\rangle = \operatorname{tr}(AB^\top).$ We use $0_{n}$ and $0_{m\times n}$ to denote the zero vector and zero matrix of size $n$ and $m\times n,$ respectively. For a general matrix $A$, $A_{k:}$ and $A_{:k}$ denote the k-th row and k-th column, respectively. We denote the matrix operator norm and Frobenius norm as $\left\|\cdot\right\|_{op}$ and $\left\|\cdot\right\|_{F}$. We use $I_d$ to denote the $d$-dimensional identity matrix and sometimes we also use $I$ when the dimension is clear from the context. For a positive semi-definite matrix $A,$ we write $\left\|x\right\|_A^2 := x^\top A x$. Unless otherwise defined, we use lower case letters for scalars and vectors, and use upper case letters for matrices.

\subsection{In-context learning} 
We begin by describing a framework for in-context learning of function classes, as initiated by~\citet{garg2022can}.  In-context learning refers to the behavior of models that operate on sequences, called \textit{prompts}, of input-output pairs $(x_1, y_1, \dots, x_N, y_N, x_\query)$, where $y_i = h(x_i)$ for some (unknown) function $h$ and examples $x_i$ and query $x_\query$.  The goal for an in-context learner is to use the prompt to form a prediction $\hat y(x_\query)$ for the query such that $\hat y(x_\query) \approx h(x_\query)$.  

From this high-level description, one can see that at a surface level, the behavior of in-context learning is no different than that of a standard learning algorithm: the learner takes as input a training dataset and returns predictions on test examples.  For instance, one can view ordinary least squares as an `in-context learner' for linear models.  However, the rather unique feature of in-context learners is that these learning algorithms can be the solutions to stochastic optimization problems defined over a distribution of prompts.  We formalize this notion in the following definition.

\begin{definition}[Trained on in-context examples]
\label{def:trained.on.incontext}
Let $\calDx$ be a distribution over an input space $\calX$, $\calH\subset \calY^ \calX$ a set of functions $\calX\to \calY$, and $\calDH$ a distribution over functions in $\calH$.  Let $\ell:\calY\times \calY \to \R$ be a loss function.   Let $\seqspace = \cup_{n\in \N} \{ (x_1, y_1, \dots, x_n, y_n): x_i \in \calX, y_i\in \calY \}$ be the set of finite-length sequences of $(x, y)$ pairs and let 
\[ \calF_\Theta = \{f_\theta : \seqspace\times \calX \to \calY,\, \theta\in \Theta\}\]
be a class of functions parameterized by $\theta$ in some set $\Theta$.  For $N>0$, we say that a model $f:\seqspace\times \calX \to \calY$ is \emph{trained on in-context examples of functions in $\calH$ under loss $\ell$ w.r.t. $(\calDH,\calD_x)$} if $f = f_{\theta^*}$ where $\theta^*\in \Theta$ satisfies
\begin{equation} \label{eq:icl.stochastic.opt.def}
\theta^* \in \mathrm{argmin}_{\theta \in \Theta} \E_{\prompt = (x_1, h(x_1), \dots, x_N, h(x_{N}), x_\query)} \l[ \ell\l(f_\theta(\prompt), h(x_\query) \r)\r],
\end{equation}
where $x_{i},x_\query \iid \calDx$ and $h\sim \calDH$ are independent.
We call $N$ the \emph{length of the prompts seen during training.}
\end{definition}

As mentioned above, this definition naturally leads to a method for \textit{learning a learning algorithm from data}: Sample independent prompts by sampling a random function $h\sim \calDH$ and feature vectors $x_i, x_{\query} \iid \calDx$, and then minimize the objective function appearing in~\eqref{eq:icl.stochastic.opt.def}
using stochastic gradient descent or other stochastic optimization algorithms.  This procedure returns a model that is learned from in-context examples and can form predictions for test (query) examples given a sequence of training data.  This leads to the following natural definition that quantifies how well such a model performs on in-context examples corresponding to a particular hypothesis class.

\begin{definition}[In-context learning of a hypothesis class]\label{def:icl.function.class}
    Let $\calDx$ be a distribution over an input space $\calX$, $\calH\subset \calY^\calX$ a class of functions $\calX\to \calY$, and $\calDH$ a distribution over functions in $\calH$.  Let $\ell: \calY\times \calY \to \R$ be a loss function.   Let $\seqspace = \cup_{n\in \N} \{ (x_1, y_1, \dots, x_n, y_n): x_i \in \calX, y_i\in \calY \}$ be the set of finite-length sequences of $(x, y)$ pairs.
We say that a model $f:\seqspace\times \calX\to \calY$ defined on prompts of the form $P=(x_1, h(x_1), \dots, x_M, h(x_M), x_{\query})$ \emph{in-context learns a hypothesis class $\calH$ under loss $\ell$ with respect to $(\calDH, \calDx)$ up to error $\eta \in \R$}
 if there exists a function $M_{\calDH,\calDx}(\eps): (0,1) \to \N$ such that for every $\eps \in (0,1)$, and for every prompt $P$ of length $M \geq M_{\calDH,\calDx}(\eps)$, 
\begin{equation}
        \E_{P = (x_1, h(x_1), \dots, x_M, h(x_M), x_\query)} \bigg[ \ell\Big( f(P), h \l(x_{\query} \r) \Big) \bigg] 
        \leq \eta + \eps,
    \end{equation}
    where the expectation is over the randomness in $x_i, x_\query \iid \calD_x$ and $h\sim \calDH$. 
\end{definition}

The additive error term $\eta$ in Definition~\ref{def:icl.function.class} above allows for the possibility that the model does not achieve arbitrarily small error.  This error could come from using a model which is not complex enough to learn functions in $\calH$ or from considering a non-realizable setting where it is not possible to achieve arbitrarily small error.

With these two definitions in hand, we can formulate the following questions: suppose a function class $\calF_\Theta$ is given and $\calDH$ corresponds to random instances of hypotheses in a hypothesis class $\calH$.  Can a model from $\calF_\Theta$ that is trained on in-context examples of functions in $\calH$ w.r.t. $(\calDH, \calDx)$ in-context learn the hypothesis class $\calH$ w.r.t. $(\calDH, \calDx)$ with small prediction error? Do standard gradient-based optimization algorithms suffice for training the model from in-context examples?  How
long must the contexts be during training and at test time to achieve small prediction error?
 In the remaining sections, we shall answer these questions for the case of one-layer transformers with linear self-attention modules when the hypothesis class is linear models, the loss of interest is the squared loss, and the marginals are (possibly anisotropic) Gaussian marginals.  

\subsection{Linear self-attention networks}
Before describing the particular transformer models we analyze in this work, we first recall the definition of the softmax-based single-head self-attention module~\citep{vaswani2017attention}.  Let $E\in \R^{d_e\times d_N}$ be an embedding matrix that is formed using a prompt $(x_1, y_1, \dots, x_N, y_N, x_{\query})$ of length $N$.
The user has the freedom to determine how this embedding matrix is formed from the prompt. One natural way to form $E$ is to stack $(x_i, y_i)^\top \in \R^{d+1}$ as the first $N$ columns of $E$ and to let the final column be $(x_\query, 0)^\top$; if $x_i\in \R^d$, $y_i\in \R$, we would then have $d_e = d+1$ and $d_N=N+1$. 
Let $\WK, \WQ \in \R^{d_k\times d_e}$ and $\WV \in \R^{d_v\times d_e}$ be the key, query, and value weight matrices, $\WP \in \R^{d_e \times d_v}$ the projection matrix, and $\rho >0$ a normalization factor. The softmax self-attention module takes as input an embedding matrix $E$ of 
width $d_N$ and outputs a matrix of the same size,
\[ f_\attn (E;\WK, \WQ, \WV, \WP) = E + \WP \WV E \cdot \softmax\l(\f{ (\WK E)^\top \WQ E}{\rho} \r),\]
where $\softmax$ is applied column-wise and, given a vector input of $v$, the $i$-th entry of $\softmax(v)$ is given by $\exp(v_i)/\sum_s \exp(v_s)$.  The $d_N \times d_N$ matrix appearing inside the softmax is referred to as the \textit{self-attention matrix}.  Note that $f_\attn$ can take as its input a sequence of arbitrary length.  

In this work, we consider a simplified version of the single-layer self-attention module, which is more amenable to theoretical analysis and yet is still capable of in-context learning linear models.  In particular, we consider a single-layer linear self-attention (LSA) model, which is a modified version of $f_\attn$ where we remove the softmax nonlinearity, merge the projection and value matrices into a single matrix $\WPV \in\R^{d_e \times d_e}$, and merge the query and key matrices into a single matrix $\WKQ \in \R^{d_e\times d_e}$.  We concatenate these matrices into $\params = (\WKQ, \WPV)$ and denote
\begin{equation} \label{eq:lsa}
    f_\lsa(E;\params) = E + \WPV E \cdot \f{ E^\top \WKQ E}{\rho}. 
\end{equation}
We note that recent theoretical works on understanding transformers looked at identical models~\citep{von2022transformers,li2023transformers,ahn2023transformers}. 
It is noteworthy that recent empirical work has shown that state-of-the-art trained vision transformers with standard softmax-based attention modules are such that $(W^K)^\top W^Q$ and $W^P W^V$ are nearly multiples of the identity matrix~\citep{trockman2023mimetic}, which can be represented under the  parameterization we consider. 

The user has the flexibility to determine the method for constructing the embedding matrix from a prompt ${P = (x_1, y_1, \dots, x_N, y_N, x_{\query})}$.  In this work, for a prompt of length $N,$ we shall use the following embedding, which stacks $(x_i, y_i)^\top \in \R^{d+1}$ into the first $N$ columns with $(x_\query, 0)^\top \in \R^{d+1}$ as the last column:
\begin{equation}
E = E(P) = \begin{pmatrix}
    x_1 & x_2 & \cdots & x_N & x_{\query} \\
    y_1 & y_2 & \cdots & y_N & 0
\end{pmatrix} \in \R^{(d+1)\times (N+1)}.\label{eq:embedding.matrix.prompt}
\end{equation}
We take the normalization factor $\rho$ to be the width of embedding matrix $E$ minus one, i.e., $\rho = d_N - 1,$ since each element in $E\cdot E^\top$ is a inner product of two vectors of length $d_N.$  
Under the above token embedding, we take $\rho = N.$  We note that there are alternative ways to form the embedding matrix with this data, e.g. by padding all inputs and labels into vectors of equal length and arranging them into a matrix \citep{akyurek2022learning}, or by stacking columns that are linear transformations of the concatenation $(x_i, y_i)$~\citep{garg2022can}, although the dynamics of in-context learning will differ under alternative parameterizations.

The network's prediction for the token $x_{\query}$
will be the bottom-right entry of matrix output by $f_\lsa$, namely,
\[ \widehat y_{\query} = \widehat y_{\query}(E;\params) = [f_{\lsa}(E; \params)]_{(d+1), (N+1)}.\] 
Here and after, we may occasionally suppress dependence on $\params$ and write $\widehat y_{\query}(E;\params)$ as $\widehat y_{\query}.$ Since the prediction takes only the right-bottom entry of the token matrix output by the LSA layer, actually only part of $\WPV$ and $\WKQ$ 
affect 
the prediction.  To see how, let us denote 
\begin{equation}\label{eqn:block_matrix}
    \WPV = \begin{pmatrix}[1.5]
    \WPV_{11} & w_{12}^{PV} \\ (w_{21}^{PV})^\top & w_{22}^{PV}\end{pmatrix}  \in \mathbb{R}^{(d+1) \times (d+1)},
    \quad
    \WKQ = \begin{pmatrix}[1.5]
    \WKQ_{11} & w_{12}^{KQ} \\ (w_{21}^{KQ})^\top & w_{22}^{KQ} \end{pmatrix}  \in \mathbb{R}^{(d+1) \times (d+1)},
\end{equation}
where $\WPV_{11} \in \mathbb{R}^{d \times d}; w_{12}^{PV},w_{21}^{PV} \in \mathbb{R}^d; w_{22}^{PV} \in \mathbb{R};$ and $\WKQ_{11} \in \mathbb{R}^{d \times d}; w_{12}^{KQ},w_{21}^{KQ} \in \mathbb{R}^d; w_{22}^{KQ} \in \mathbb{R}.$ Then, the prediction $\widehat y_{\query}$ is 
\begin{equation}\label{eqn:simple_expression_prediction}
    \widehat y_{\query} = 
    \begin{pmatrix}[1.5]
     (w_{21}^{PV})^\top & w_{22}^{PV}\end{pmatrix}
     \cdot \left(\frac{EE^\top}{N}\right) 
     \begin{pmatrix}[1.5]
    \WKQ_{11} \\ (w_{21}^{KQ})^\top \end{pmatrix}
    x_{\query},
\end{equation}
since only the last row of $\WPV$ and the first $d$ columns of $\WKQ$ 
affects 
the prediction, which means we can simply take all other entries zero in the following sections.

\subsection{Training procedure}

In this work, we will consider the task of in-context learning linear predictors.  We will assume training prompts are sampled as follows. Let $\Lambda$ be a positive definite covariance matrix.  Each training prompt, indexed by $\tau \in \N$, takes the form of $P_\tau = (x_{\tau, 1}, h_\tau(x_{\tau_1}), \dots, x_{\tau, N}, h_\tau(x_{\tau, N}), x_{\tau, \query})$, where task weights $w_\tau \iid \normal (0, I_d)$, inputs $x_{\tau, i}, x_{\tau, \query} \iid \normal(0, \Lambda)$, and labels $h_{\tau}(x) = \sip{w_\tau}{x}$.  
 
Each prompt corresponds to an embedding matrix $E_\tau$, formed using the transformation~\eqref{eq:embedding.matrix.prompt}:
\begin{equation*}
    E_\tau := \begin{pmatrix}
    x_{\tau,1} & x_{\tau,2} & \cdots & x_{\tau,N} & x_{\tau,\query} \\
    \sip{w_\tau}{x_{\tau,1}} & \sip{w_\tau}{x_{\tau,2}} & \cdots & \sip{w_\tau}{x_{\tau,N}} & 0
\end{pmatrix} \in \R^{(d+1)\times (N+1)}.
\end{equation*}
We denote the prediction of the LSA model on the query label in the task $\tau$ as $\widehat y_{\tau,\query}$, which is the bottom-right element of $f_\lsa(E_\tau),$ where $f_{\lsa}$ is the linear self-attention model defined in \eqref{eq:lsa}.  The empirical risk over $B$ independent prompts is defined as
\begin{equation}\label{eqn:emp_loss_maintext}
    \widehat L(\params) = \f{1}{2B} \sum_{\tau=1}^B \bigg(\widehat y_{\tau,\query} - \sip{w_\tau}{x_{\tau,\query}}\bigg)^2.
\end{equation}
We shall consider the behavior of gradient flow-trained networks over the population loss induced by the limit of infinite training tasks/prompts $B\to \infty$:
\begin{equation}\label{eqn:population_loss}
    L(\params) = \lim_{B\to \infty} \widehat L(\params) = \f 12 \E_{w_\tau, x_{\tau,1}, \cdots, x_{\tau,N}, x_{\tau,\query}} \l[ (\widehat y_{\tau,\query} - \sip{w_\tau}{x_{\tau,\query}})^2 \r]
\end{equation}
Above, the expectation is taken w.r.t. the covariates $\{x_{\tau,i}\}_{i=1}^N \cup \{x_{\query}\}$ in the prompt and the weight vector $w_\tau$, i.e. over $x_{\tau,i},x_\query \iid \normal(0, \Lambda)$ and $w_{\tau} \sim \normal(0,I_d)$.  
Gradient flow captures the behavior of gradient descent with infinitesimal step size and has dynamics given by the following differential equation:
\begin{equation}\label{eqn:gf}
    \f{\mathrm{d}}{\mathrm dt} \params = - \nabla L(\params).
\end{equation}
We will consider gradient flow with an initialization that satisfies the following. 
\begin{assumption}[Initialization]\label{assume_init}
    Let $\sigma>0$ be a parameter, and let $\Theta \in \R^{d\times d}$ be any matrix satisfying $\snorm{\Theta \Theta^\top}_F =1$ and $\Theta \Lambda \neq 0_{d\times d}$. We assume
\begin{equation} \label{eq:wpv.wkq.init}
    \WPV(0)= 
    \sigma\begin{pmatrix}
    0_{d\times d} & 0_d \\
    0_d^\top & 1
    \end{pmatrix},\quad 
    \WKQ(0) = 
    \sigma \begin{pmatrix}
    \Theta \Theta^\top & 0_d \\ 
    0_d^\top & 0 
    \end{pmatrix}.
\end{equation}
\end{assumption}
This initialization is satisfied for a particular class of random initialization schemes: if $M$ has i.i.d. entries from a continuous distribution, then by setting $\Theta \Theta^\top = MM^\top / \snorm{M M^\top}_F$, the assumption is satisfied almost surely.  The reason we use this particular initialization scheme will be made more clear in Section~\ref{sec:proof} when we describe the proof, but at a high-level this is due to the fact that the predictions~\eqref{eqn:simple_expression_prediction} can be viewed as the output of a two-layer linear network, and initializations satisfying Assumption~\ref{assume_init} allow for the layers to be `balanced' throughout the gradient flow trajectory.  Random initializations that induce this balancedness condition have been utilized in a number of theoretical works on deep linear networks~\citep{du2018algorithmic,arora2018optimization,arora2019implicit,azulay2021implicit}. 
We leave the question of convergence under alternative random initialization schemes for future work.

\section{Main results}
In this section, we present the main results of this paper. First, in Section~\ref{sec:main.convergence.error}, we prove the gradient flow on the population loss will converge to a specific global optimum.  We characterize the prediction error of the trained transformer at this global minimum when given a prompt from a new prediction task.  Our characterization allows for the possibility that this new prompt comes from a nonlinear prediction task.  
We then instantiate our results for well-specified linear regression prompts and characterize the number of samples needed to achieve small prediction error, showing that transformers can in-context learn linear models 
when trained on 
in-context examples of linear models. 

Next, in Section~\ref{sec:main.distshift}, we analyze the behavior of the trained transformer under a variety of distribution shifts.  We show the transformer is robust to a number of distribution shifts, including task shift (when the labels in the prompt are not deterministic linear functions of their input) and query shift (when the query example $x_\query$ has a possibly different distribution than the test prompt).  On the other hand, we show that the transformer suffers from covariate distribution shifts, i.e. when the training prompt covariate distribution differs from the test prompt covariate distribution.

Finally, motivated by the failure of the trained transformer under covariate distribution shift, we consider in Section~\ref{sec:main.randomcov} the setting of training on in-context examples with varying covariate distributions across prompts.  We prove that transformers with a single linear self-attention layer trained by gradient flow converge to a global minimum of the population objective, but that the trained transformer still fails to perform well on new prompts. We complement our proof in the linear self-attention case with experiments on large, nonlinear transformer architectures which we show
 are more robust under covariate shifts. 

\subsection{Convergence of gradient flow and prediction error for new tasks}
\label{sec:main.convergence.error}
First, we prove that under suitable initialization, gradient flow will converge to a global optimum.

\begin{restatable}[Convergence and limits]{theorem}{mainresult}\label{thm:mainresult}
Consider gradient flow of the linear self-attention network $f_\lsa$ defined in~\eqref{eq:lsa} over the population loss~\eqref{eqn:population_loss}.  Suppose the initialization satisfies Assumption~\ref{assume_init} with initialization scale $\sigma>0$ satisfying $\sigma^2\snorm{\Gamma}_{op}\sqrt d <2$ 
where we have defined
\[ \Gamma := \left(1 + \frac{1}{N}\right)\Lambda + \frac{1}{N}\operatorname{tr}(\Lambda) I_d \in \mathbb{R}^{d \times d}.\]
Then gradient flow converges to a global minimum of the population loss \eqref{eqn:population_loss}. Moreover, $\WPV$ and $\WKQ$ converge to $\WPV_*$ and $\WKQ_*$ respectively, where
\begin{equation}\label{eq:limit_expression}
\begin{aligned}
    \WKQ_* &=
    \left[\operatorname{\tr}\left(\Gamma^{-2}\right)\right]^{-\frac{1}{4}}
    \begin{pmatrix}[1.5]
        \Gamma^{-1} & 0_{d} \\
        0_d^\top & 0
    \end{pmatrix}, \qquad
   \WPV_* =
    \left[\operatorname{\tr}\left(\Gamma^{-2}\right)\right]^{\frac{1}{4}}
    \begin{pmatrix}[1.5]
        0_{d \times d} & 0_{d} \\
        0_d^\top & 1
    \end{pmatrix}.
\end{aligned}
\end{equation}
\end{restatable}

The full proof of this theorem appears in Appendix~\ref{appendix:proof_main}. 
We note that if we restrict our setting to $\Lambda=I_d$, then the limiting solution described found by gradient flow is quite similar to the construction of~\citet{von2022transformers}.  Since the prediction of the transformer is the same if we multiply $\WPV$ by a constant $c\neq 0$ and simultaneously multiply $\WKQ$ by $c^{-1}$, the only difference (up to scaling) is that the top-left entry of their $\WKQ$ matrix is $I_d$ rather than the $(1 + (d+1)/N)^{-1}I_d$ that we find for the case $\Lambda=I_d$.

Next, we would like to characterize the prediction error of the trained network described above when the network is given a new prompt.  
Let us consider a prompt of the form $(\testx_1, \sip{w}{\testx_1}, \dots, \testx_M, \sip{w}{\testx_M}, x_\query)$ where $w\in \R^d$ and $\testx_i, x_\query \iid \normal(0,\Lambda)$.  
A simple calculation shows that the prediction $\hat y_\query$ at the global optimum with parameters $\WKQ_*$ and $\WPV_*$ is given by \begin{align*}
    \widehat y_{\query} 
    &= \begin{pmatrix}[1.5]
    0_d^\top & 1
\end{pmatrix}
\begin{pmatrix}[1.5]
    \frac{1}{M}\sum_{i=1}^{M} \testx_i \testx_i^\top + \frac{1}{M}\testx_{\query} \testx_{\query}^\top & \frac{1}{M}\sum_{i=1}^{M} \testx_i \testx_i^\top w \\
    \frac{1}{M}\sum_{i=1}^{M} w^\top \testx_i \testx_i^\top  &
    \frac{1}{M}\sum_{i=1}^{M} w^\top \testx_i \testx_i^\top w
\end{pmatrix}
\begin{pmatrix}[1.5]
    \Gamma^{-1} & 0_d \\
    0_d^\top & 0
\end{pmatrix}
\begin{pmatrix}[1.5]
    \testx_\query \\
    0
\end{pmatrix} \notag\\
&=\testx_\query^\top  \Gamma^{-1} \l( \f 1 M \summ i M \testx_i \testx_i^\top\r) w. \numberthis \label{eq:prediction.trained.transformer.linear}
\end{align*}
When the length of prompts seen during training $N$ is large, $\Gamma^{-1} \approx \Lambda^{-1}$, and when the test prompt length $M$ is large, $\f 1 M \summ i M x_i x_i^\top \approx \Lambda$, so that $\hat y_\query \approx x_\query^\top w$.  Thus, for sufficiently large prompt lengths, \textit{the trained transformer indeed in-context learns the class of linear predictors}. 

In fact, we can generalize the above calculation for test prompts which could take a significantly different form than the training prompts.  Consider prompts that are of the form $(\testx_1, \testy_1, \dots, \testx_n, \testy_n, \testx_\query)$ where, for some joint distribution $\calD$ over $(x,y)$ pairs with marginal distribution $x\sim \normal(0,\Lambda)$, we have $(\testx_i, \testy_i) \iid \calD$ and $x_\query\sim \normal(0,\Lambda)$ independently.  Note that this allows for a label $\testy_i$ to be a nonlinear function of the input $\testx_i$.  The prediction of the trained transformer for this prompt is then
\begin{align}
    \widehat y_{\query} 
    &= \begin{pmatrix}[1.5]
    0_d^\top & 1
\end{pmatrix}
\begin{pmatrix}[1.5]
    \frac{1}{M}\sum_{i=1}^{M} \testx_i \testx_i^\top + \frac{1}{M}\testx_{\query} \testx_{\query}^\top & \frac{1}{M}\sum_{i=1}^{M} \testx_i \testy_i \\
    \frac{1}{M}\sum_{i=1}^{M} \testx_i^\top \testy_i &
    \frac{1}{M}\sum_{i=1}^{M} \testy_i^2
\end{pmatrix}
\begin{pmatrix}[1.5]
    \Gamma^{-1} & 0_d \\
    0_d^\top & 0
\end{pmatrix}
\begin{pmatrix}[1.5]
    \testx_\query \\
    0
\end{pmatrix} \notag\\
&=\testx_\query^\top  \Gamma^{-1} \l( \f 1 M \summ i M \testy_i \testx_i\r). \label{eq:prediction.trained.transformer}
\end{align}
Just as before, when $N$ is large we have $\Gamma^{-1} \approx \Lambda^{-1}$, and so when $M$ is large as well this implies
\begin{equation}\label{eq:prediction.trained.transformer.largeMN}
 \hat y_\query \approx \testx_\query^\top \Lambda^{-1} \E_{(\testx, \testy) \sim \calD}[yx] = x_\query^\top \l( \argmin_{w\in \R^d} \E_{(\testx, \testy) \sim \calD}[(y - \sip wx)^2] \r) .
\end{equation} 
This suggests that trained transformers in-context learn the \textit{best linear predictor} over a distribution when the test prompt consists of i.i.d. samples from a joint distribution over feature-response pairs.   In the following theorem, we formalize the above and characterize the prediction error
when prompts take this form. 

\begin{restatable}{theorem}{generalrisk}\label{thm:general_risk_maintext}
    Let $\calD$ be a distribution over $(x,y) \in \R^d \times \R,$ whose marginal distribution on $x$ is $\calDx = \normal(0,\Lambda).$ 
Assume $\E_\calD[y], \E_\calD[xy], \E_\calD[y^2 xx^\top]$ exist and are finite.
    Assume the test prompt is of the form $\testprompt = (\testx_1, \testy_1, \dots, \testx_M, \testy_M, \testx_{\query}),$ where $(\testx_i,\testy_i), (\testx_\query,\testy_\query) \iid \calD.$
    Let $f^*_{\lsa}$ be the LSA model with parameters $\WPV_*$ and $\WKQ_*$ in \eqref{eq:limit_expression}, and $\widehat y_\query$ is the prediction for $\testx_\query$ given the prompt. If we define
    \begin{equation}
        a := \Lambda^{-1} \E_{(x,y) \sim \calD} \left[xy\right],\qquad 
        \Sigma := \E_{(x,y) \sim \calD}  \left[\big(xy  - \E \left(x y\right) \big) \big(x y  - \E \left(x y\right) \big)^\top\right],
    \end{equation}
    then, for $\Gamma = \Lambda + \frac{1}{N} \Lambda + \frac{1}{N}\tr(\Lambda)I_d.$ we have,
    \begin{align}\label{eqn:risk_nonlinear}
        \E \left(\widehat{y}_{\query} - \testy_\query\right)^2
        & = \underbrace{\min_{\testw \in \mathbb{R}^d} \E \left(\sip{\testw}{\testx_\query} - \testy_\query\right)^2}_{\text{Error of best linear predictor}} \notag \\*
        &+ \frac{1}{M} \tr\left[\Sigma \Gamma^{-2} \Lambda\right]
        + \frac{1}{N^2} \left[ \left\|a\right\|^2_{\Gamma^{-2}\Lambda^3} 
        + 2\tr(\Lambda) \left\|a\right\|^2_{\Gamma^{-2}\Lambda^2} + \tr(\Lambda)^2 \left\|a\right\|^2_{\Gamma^{-2}\Lambda} \right],
    \end{align}
    where the expectation is over $(\testx_i,\testy_i), (\testx_\query,\testy_\query) \iid \calD$.  
\end{restatable}
The full proof is deferred to Appendix \ref{appendix_proof_risk_general}.  Let us now make a few remarks on the above theorem before considering particular instances of $\calD$ where we may provide more explicit bounds on the prediction error.  

First, this theorem shows that, provided the length of prompts seen during training ($N$) and the length of the test prompt ($M$) is large enough, a transformer trained by gradient flow from in-context examples achieves prediction error competitive with the best linear model.  
Next, our bound shows that the length of prompts seen during training and the length of prompts seen at test-time have different effects on the prediction error: ignoring dimension and covariance-dependent factors, the prediction error is at most $O(1/M+1/N^2)$, decreasing more rapidly as a function of the training prompt length $N$ compared to the test prompt length $M$.  Additionally, it is worth noting that even if $M\to \infty$, the gap between the prediction error of the transformer with that of the best linear predictor does not vanish unless $N\to \infty$ as well.  Thus, the transformer is inherently limited by training on finite-length prompts.

Let us now consider when $\calD$ corresponds to noiseless linear models, so that for some $\testw\in \R^d$, we have $(\testx, \testy) = (\testx, \sip{\testw}{\testx})$, in which case the prediction of the trained transformer is given by~\eqref{eq:prediction.trained.transformer.linear}. 
Moreover, a simple calculation shows that the $\Sigma$ from Theorem~\ref{thm:general_risk_maintext} takes the form $\Sigma = \snorm{\testw}_\Lambda^2 \Lambda + \Lambda \testw \testw^\top \Lambda$. Hence Theorem~\ref{thm:general_risk_maintext} implies the prediction error for the prompt $\testprompt=(x_1, \sip{w}{x_1}, \dots, x_M, \sip{w}{x_M}, x_\query)$ is
    \begin{align*}
        &\quad \E_{\testx_1,...,\testx_M,\testx_{\query}} \left(\widehat{y}_{\query} - \sip{\testw}{\testx_{\query}}\right)^2 \notag \\ 
        &= \frac{1}{M} \left\{\left\|\testw\right\|^2_{\Gamma^{-2} \Lambda^3} + \tr(\Gamma^{-2} \Lambda^2) \left\|\testw\right\|^2_\Lambda\right\} 
        +\frac{1}{N^2} \left\{\left\|\testw\right\|^2_{\Gamma^{-2} \Lambda^3} + 2 \left\|\testw\right\|^2_{\Gamma^{-2} \Lambda^2} \tr(\Lambda) + \left\|\testw\right\|^2_{\Gamma^{-2} \Lambda} \tr(\Lambda)^2\right\} \\
        & \leq \frac{d+1}{M} \left\|\testw\right\|^2_{\Lambda} + \frac{1}{N^2} \left[\left\|\testw\right\|^2_{\Lambda} + 2 \left\|\testw\right\|^2_{2} \tr(\Lambda) + \left\|\testw\right\|^2_{\Lambda^{-1}} \tr(\Lambda)^2\right],
    \end{align*}
The inequality above uses that $\Gamma \succ \Lambda$.  Finally, if we assume that $\testw \sim \normal(0, I_d)$ and denote $\kappa$ as the condition number of $\Lambda$, then by taking expectations over $\testw$ we get the following:
    \begin{align*}
        \E_{\testx_1,...,\testx_M,\testx_{\query},\testw } \left(\widehat{y}_{\query} - \sip{\testw}{\testx_{\query}}\right)^2 &\leq \frac{(d+1)\tr(\Lambda)}{M} + \frac{1}{N^2} \left[\tr(\Lambda) + 2d \tr(\Lambda) + \tr(\Lambda^{-1}) \tr(\Lambda)^2\right]\\
        &\leq \frac{(d+1)\tr(\Lambda)}{M} + \frac{(1 + 2d + d^2\kappa)\tr(\Lambda)}{N^2},
    \end{align*}
From the upper bound above, we can see the rate w.r.t $M$ and $N$ are still at most $O(1/M)$ and $O(1/N^2)$ respectively. Moreover, the generalization error also scales with dimension $d$, $\tr(\Lambda)$ and the condition number $\kappa$. This suggests that for in-context examples involving covariates of greater variance, or a more ill-conditioned covariance matrix, the generalization error will be higher for the same lengths of training and testing prompts. Putting the above together with Theorem~\ref{thm:general_risk_maintext}, Definition~\ref{def:trained.on.incontext} and Definition~\ref{def:icl.function.class}, we get the following corollary.
\begin{corollary}
\label{corollary:trained.transformer}
The transformer $f_\lsa$ 
trained on length-$N$ prompts of in-context examples of functions in $\{x\mapsto \sip wx \}$ w.r.t. $w\sim \normal(0, I_d)$ and $\calD_x = \normal(0, \Lambda)$ by gradient flow on the population loss~\eqref{eqn:population_loss} for initializations satisfying Assumption~\ref{assume_init} converges to the model $f_\lsa(\cdot\ ; \WKQ_*, \WPV_*)$.  This model takes a prompt $\testprompt = (\testx_1, \testy_1, \dots, \testx_M, \testy_M, \testx_\query)$ and returns a prediction $\hat y_\query$ for $x_\query$ given by 
\[ \hat y_\query = [f_\lsa(\testprompt; \WKQ_*, \WPV_*)]_{d+1,M+1} = \testx_\query^\top \l( \Lambda + \f 1 N \Lambda + \f {\tr(\Lambda)} N I_d\r)^{-1} \l( \f 1 M \summ i M \testy_i \testx_i\r).\]
 This model in-context learns the class of linear models $\{x\mapsto \sip wx\}$ with respect to $w\sim \normal(0,I_d)$ and $\calDx = \normal(0,\Lambda)$ up to error $\eta := (1+2d+d^2\kappa) \tr(\Lambda)/N^2$ (where $\kappa$ is the condition number of $\Lambda$): provided $M\geq (d+1)\tr(\Lambda)\eps^{-1}$, the model achieves prediction error at most $\eta+\eps$.
\end{corollary}
It is worth emphasizing that the transformer $f_\lsa(\cdot\ ; \WKQ_*, \WPV_*)$ only learns the function class up to error $\eta=O(1/N^2)$ in the sense of Definition~\ref{def:icl.function.class}.  In particular, training on finite-length prompts leads to prediction error bounded away from zero.

\subsection{Behavior of trained transformer under distribution shifts}\label{sec:main.distshift}
Using the identity~\eqref{eq:prediction.trained.transformer}, it is straightforward to characterize the behavior of the trained transformer under a variety of distribution shifts.  In this section, we shall examine a number of shifts that were first explored empirically for transformer architectures by~\citet{garg2022can}.  Although their experiments were for transformers trained by gradient descent, we find that (in the case of linear models) many of the behaviors of the trained transformers under distribution shift are identical to those predicted by our theoretical characterizations of the performance of transformers with a single linear self-attention layer trained by gradient flow on the population. 

Following~\citet{garg2022can}, for training prompts of the form $(x_1, h(x_1), \dots, x_N, h(x_N), x_\query)$, let us assume $x_i, x_\query \iid \dtrainx$ and $h\sim \dtrainh$, while for test prompts let us assume $x_i\iid \dtestx$, $x_\query \sim \dtestquery$, and $h\sim \dtesth$.  We will consider the following distinct categories of shifts:

\begin{itemize}
\item Task shifts: $\dtrainh\neq \dtesth$.
\item Query shifts: $\dtestquery \neq \dtestx$.  
\item Covariate shifts: $\dtrainx \neq \dtestx$.
\end{itemize} 

In the following, we shall fix $\dtrainx = \normal(0,\Lambda)$ and vary the other distributions.  Recall from~\eqref{eq:prediction.trained.transformer} that the prediction for a test prompt $(x_1, y_1, \dots, x_N, y_N, x_\query)$ is given by (for $N$ large),
\begin{equation}\label{eq:hatyquery.nlarge}
\hat y_\query = x_\query^\top\Gamma^{-1} \l( \f 1 M \summ i M y_i x_i \r) \approx x_\query^\top \Lambda^{-1} \l( \f 1 M \summ i M y_i x_i \r).
\end{equation}

\paragraph{Task shifts.}  These shifts are tolerated easily by the trained transformer.  As Theorem~\ref{thm:general_risk_maintext} shows, the trained transformer is competitive with the best linear model provided the prompt length during training and at test time is large enough.  In particular, even if the prompt is such that the labels $y_i$ are not given by $\sip{w}{x_i}$ for some $w\sim \normal(0,I_d)$, the trained transformer will compute a prediction which has error competitive with the best linear model that fits the test prompt. 

For example, consider a prompt corresponding to a noisy linear model, so that the prompt consists of a sequence of $(x_i, y_i)$ pairs where $y_i = \sip{w}{x_i}+\eps_i$ for some arbitrary vector $w\in \R^d$ and independent sub-Gaussian noise $\eps_i$.  Then from~\eqref{eq:hatyquery.nlarge}, the prediction of the transformer on query examples is
\[ \hat y_\query \approx x_\query^\top \Lambda^{-1} \l( \f 1 M \summ i M y_i x_i \r) = x_\query^\top \Lambda^{-1} \l( \f 1 M \summ i M x_i x_i^\top\r) w + x_\query^\top \Lambda^{-1} \l( \f 1 M \summ i M \eps_i x_i\r).\]
Since $\eps_i$ is mean zero and independent of $x_i$, this is approximately $x_\query^\top w$ when $M$ is large.  And note that this calculation holds for an \textit{arbitrary} vector $w$, not just those which are sampled from an isotropic Gaussian or those with a particular norm.  This behavior coincides with that of the trained transformers observed by~\citet{garg2022can}. 

\paragraph{Query shifts.}  Continuing from~\eqref{eq:hatyquery.nlarge}, since $y_i = \sip{w}{x_i}$,
\[ \hat y_\query \approx x_\query^\top \Lambda^{-1} \l( \f 1 M \summ i M x_i x_i^\top \r) w.\]
From this we see that whether query shifts can be tolerated hinges upon
the distribution of the $x_i$'s.  Since $\dtrainx = \dtestx$, if $M$ is large then
\begin{equation}\label{eqn:approx-pred}
\hat y_\query \approx x_\query^\top \Lambda^{-1} \Lambda w = x_\query^\top w.
\end{equation}
Thus, very general shifts in the query distribution can be tolerated. 
On the other hand, very different behavior can be expected if 
$M$ is not large and the query example depends on the training data.  For example, if the query example is orthogonal to the subspace spanned by the $x_i$'s, the prediction will be zero,
as was observed with transformer architectures by~\citet{garg2022can}. 

\paragraph{Covariate shifts.}  In contrast to task and query shifts, covariate shifts cannot be fully tolerated in the transformer.  This can be easily seen due to the identity~\eqref{eq:prediction.trained.transformer}: when $\dtrainx\neq \dtestx$, then the approximation in~\eqref{eqn:approx-pred} does not hold as $\f 1 M \summ i M x_i x_i^\top$ will not cancel $\Gamma^{-1}$ when $M$ and $N$ are large.  For instance, if we consider test prompts where the covariates are scaled by a constant $c\neq 1$, then 
\[ \hat y_\query \approx x_\query^\top \Lambda^{-1} \l( \f 1 M \summ i M x_i x_i^\top \r) \approx x_\query^\top \Lambda^{-1} c^2 \Lambda w = c^2 x_\query^\top w \neq x_\query^\top w.\]
This failure mode of the trained transformer with linear self-attention was also observed in the trained transformer architectures by~\citet{garg2022can}.  This suggests that although the predictions of the transformer may look similar to those of ordinary least squares in some settings, the algorithm implemented by the transformer is not the same since ordinary least squares is robust to scaling of the features by a constant.  

It may seem surprising that a transformer trained on linear regression tasks fails in settings where ordinary least squares performs well.  However, both the linear self-attention transformer we consider and the transformers considered by~\citet{garg2022can} were trained on instances of linear regression when the covariate distribution $\calDx$ over the features was fixed across instances.  This leads to the natural question of what happens if the transformers instead are trained on prompts where the covariate distribution varies across instances, which we explore in the following section.

\subsection{Transformers trained on prompts with random covariate distributions}\label{sec:main.randomcov}
In this section, we will consider a variant of training on in-context examples (in the sense of Definition~\ref{def:trained.on.incontext}) where the distibution $\calD_x$ is itself sampled randomly from a distribution, and training prompts are of the form $(x_1, h(x_1), \dots, x_N, h(x_N), x_\query)$ where $x_i, x_\query \iid \calD_x$ and $h\sim \calDH$.   More formally, we can generalize Definition~\ref{def:trained.on.incontext} as follows.
\begin{definition}[Trained on in-context examples with random covariate distributions]
\label{def:trained.on.incontext.random}
Let $\Delta$ be a distribution over distributions $\calDx$ defined on an input space $\calX$, $\calH\subset \calY^ \calX$ a set of functions $\calX\to \calY$, and $\calDH$ a distribution over functions in $\calH$.  Let $\ell:\calY\times \calY \to \R$ be a loss function.   Let $\seqspace = \cup_{n\in \N} \{ (x_1, y_1, \dots, x_n, y_n): x_i \in \calX, y_i\in \calY \}$ be the set of finite-length sequences of $(x, y)$ pairs and let 
\[ \calF_\Theta = \{f_\theta : \seqspace\times \calX \to \calY,\, \theta\in \Theta\}\]
be a class of functions parameterized by some set $\Theta$.  We say that a model $f:\seqspace\times \calX \to \calY$ is \emph{trained on in-context examples of functions in $\calH$ under loss $\ell$ w.r.t. $\calDH$ and distribution over covariate distributions $\Delta$} if $f = f_{\theta^*}$ where $\theta^*\in \Theta$ satisfies
\begin{equation} \label{eq:icl.stochastic.opt.def.random}
\theta^* \in \mathrm{argmin}_{\theta \in \Theta} \E_{\prompt = (x_1, h(x_1), \dots, x_N, h(x_{N}), x_\query)} \l[ \ell\l(f_\theta(\prompt), h(x_\query) \r)\r],
\end{equation}
where $\calDx \sim \Delta$, $x_{i},x_\query \iid \calDx$ and $h\sim \calDH$.
\end{definition}
We recover the previous definition of training on in-context examples by taking $\Delta$ to be concentrated on a singleton, $\supp(\Delta) = \{ \calD_x\}$.  The natural question is then, if a model $f$ is trained on in-context examples from a function class $\calH$ w.r.t. $\calDH$ and a \textit{distribution} $\Delta$ over covariate distributions, and if one then samples some covariate distribution $\calDx\sim \Delta$, does $f$ in-context learn $\calH$ w.r.t. $(\calDH, \calDx)$ for that $\calDx$ (cf. Definition~\ref{def:icl.function.class}) with small prediction error? Since $\calDx$ is random, we can hope that this may hold in expectation or with high probability over the sampling of the covariate distribution.  In the remainder of this section, we will explore this question for transformers with a linear self-attention layer trained by gradient flow on the population loss.

We shall again consider the case where the covariates have Gaussian marginals,  $x_i \sim \normal(0, \Lambda)$, but we shall now assume that within each prompt we first sample a random covariance matrix $\Lambda$.  For simplicity, we will restrict our attention to the case where $\Lambda$ is diagonal.  More formally, we shall assume training prompts are sampled as follows.  For each independent task indexed by $\tau\in [B]$, we first sample $w_\tau \sim \normal(0, I_d)$.  Then, for each task $\tau$ and coordinate $i\in [d]$, we sample $\lambda_{\tau, i}$ independently such that the distribution of each $\lambda_{\tau,i}$ is fixed and has finite third moments and is strictly positive almost surely.  We then form a diagonal matrix
\[ \Lambda_\tau = \diag(\lambda_{\tau, 1}, \dots, \lambda_{\tau, d}).\]
Thus the diagonal entries of $\Lambda_\tau$ are independent but could have different distributions, and $\Lambda_\tau$ is identically distributed for $\tau =1, \dots, B$.  
Then, conditional on $\Lambda_\tau$, we sample independent and identically distributed $x_{\tau,1}, \dots, x_{\tau, N}, x_{\tau,\query}\sim \normal(0, \Lambda_\tau)$.  A training prompt is then given by $P_\tau = (x_{\tau, 1}, \sip{w_\tau}{x_{\tau,1}}, \dots, x_{\tau, N}, \sip{w_\tau}{x_{\tau, N}}, x_{\tau, \query})$  Notice that here, $x_{\tau,i}, x_{\tau,\query}$ are conditionally independent given the covariance matrix $\Lambda_\tau$, but not independent in general.  We consider the same token embedding matrix as \eqref{eq:embedding.matrix.prompt} and linear self-attention network, which forms the prediction $\hat y_{\query,\tau}$ as in~\eqref{eqn:simple_expression_prediction}.  
The empirical risk is the same as before (see~\eqref{eqn:emp_loss_maintext}), and as in~\eqref{eqn:population_loss}, we then take $B\to\infty$ and consider the gradient flow on the population loss.   The population loss now includes an expectation over the distribution of the covariance matrices in addition to the task weight $w_\tau$ and covariate distributions, and is given by
\begin{equation}\label{eqn:loss_random_cov}
    L(\params) = \f 12 \E_{w_\tau, \Lambda_{\tau}, x_{\tau,1}, \cdots, x_{\tau,N}, x_{\tau,\query}} \l[ (\widehat y_{\tau,\query} - \sip{w_\tau}{x_{\tau,\query}})^2 \r].
\end{equation}

\

In the main result for this section, we show that gradient flow with a suitable initialization converges to a global minimum, and we characterize the limiting solution.  The proof will be deferred to Appendix \ref{appendix_proof_random}.

\begin{theorem}[Global convergence in random covariance case]\label{thm:convergence_random}
    Consider gradient flow of the linear self-attention network $f_\lsa$ defined in \eqref{eq:lsa} over the population loss \eqref{eqn:loss_random_cov}, where  $\Lambda_\tau$ are diagonal with independent diagonal entries which are strictly positive a.s. and have finite third moments. Suppose the initialization satisfies Assumption \ref{assume_init}, $\left\|\E \Lambda_\tau \Theta\right\|_F \neq 0$, with initialization scale $\sigma>0$ satisfying
\begin{equation}\label{eqn:init_cond_random}
     \sigma^2 < \frac{2\left\|\E \Lambda_\tau \Theta\right\|_F^2}{\sqrt{d} \left[\E \left\|\Gamma_\tau\right\|_{op} \left\|\Lambda_\tau\right\|_F^2 \right]}.
    \end{equation}
    Then gradient flow converges to a global minimum of the population loss \eqref{eqn:loss_random_cov}. Moreover, $\WPV$ and $\WKQ$ converge to $\WPV_*$ and $\WKQ_*$ respectively, where
    \begin{equation}\label{eq:limit_expression_random}
    \begin{aligned}
    \WKQ_*
    &= \left\|\left[\E \Gamma_\tau \Lambda_\tau^2\right]^{-1} \E \left[\Lambda_\tau^2\right]\right\|_F^{-\frac{1}{2}} \cdot  
    \begin{pmatrix}[1.5]
        \left[\E \Gamma_\tau \Lambda_\tau^2\right]^{-1} \left[\E \Lambda_\tau^2\right] & 0_d \\
        0_d^\top & 0
    \end{pmatrix}, \\
    \WPV_*
    &= \left\|\left[\E \Gamma_\tau \Lambda_\tau^2\right]^{-1} \E \left[\Lambda_\tau^2\right]\right\|_F^{\frac{1}{2}} \cdot
    \begin{pmatrix}[1.5]
    0_{d \times d} & 0_d \\
    0_d^\top & 1
    \end{pmatrix},
\end{aligned}
\end{equation}
where $\Gamma_\tau = \frac{N+1}{N} \Lambda_\tau + \frac{1}{N}\tr(\Lambda_\tau)I_d \in \R^{d \times d}$ and the expectations above are over the distribution of $\Lambda_\tau.$ 
\end{theorem}

\ 
\def\new{\mathsf{new}}
From this result, we can see why the trained transformer fails in the random covariance case. Suppose we have a new prompt corresponding to a weight matrix $w \in \R^d$ and covariance matrix $\Lambda_\new,$ sampled from the same distribution as the covariance matrices for training prompts, so that conditionally on $\Lambda_\new$ we have $x_i, x_\query \iid \normal(0, \Lambda_\new)$.  The ground-truth labels are given by $y_i = \sip{w}{x_i}, i \in [M]$ and $y_\query = \sip{w}{x_\query}.$ At convergence, the prediction by the trained transformer on the new task will be
\begin{align}
    \widehat y_{\query} 
    &= \begin{pmatrix}[1.5]
    0_d^\top & 1
\end{pmatrix}
\begin{pmatrix}[1.5]
    \frac{1}{M}\sum_{i=1}^{M} x_i x_i^\top + \frac{1}{M} x_{\query} x_{\query}^\top & \frac{1}{M}\sum_{i=1}^{M} x_i y_i \\
    \frac{1}{M}\sum_{i=1}^{M} x_i^\top \testy_i &
    \frac{1}{M}\sum_{i=1}^{M} y_i^2
\end{pmatrix}
\begin{pmatrix}[1.5]
    \left[\E \Gamma_\tau \Lambda_\tau^2\right]^{-1} \l[ \E \Lambda_\tau^2\r] & 0_d\\
    0_d^\top & 0
\end{pmatrix}
\begin{pmatrix}[1.5]
x_{\query} \\
0
\end{pmatrix}
\notag\\
    &= x_{\query}^\top \cdot \left[\E \Lambda_\tau^2\right]  \left[\E \Gamma_\tau \Lambda_\tau^2\right]^{-1} \cdot \left[\frac{1}{M}\sum_{i=1}^{M} x_i x_i^\top \right] w \notag\\
    & \to x_{\query}^\top \cdot \left[\E \Lambda_\tau^2\right]  \left[\E \Gamma_\tau \Lambda_\tau^2\right]^{-1} \cdot \Lambda_\new w \quad \text{ almost surely when } M \to \infty.\label{eq:yhatquery.randomcov.Minfinity}
\end{align}
The last line comes from the strong law of large numbers.  Thus, in order for the prediction on the query example to be close to the ground-truth $x_\query^\top w$, we need $\left[\E \Lambda_\tau^2\right]  \left[\E \Gamma_\tau \Lambda_\tau^2\right]^{-1} \cdot \Lambda_\new$ to be close to the identity.  When $\Lambda_\tau\equiv \Lambda_\new$ is deterministic, this indeed is the case as we know from Theorem~\ref{thm:general_risk_maintext}.  However, this clearly does not hold in general when $\Lambda_\tau$ is random. 

To make things concrete, let us assume for simplicity that $M,N \to \infty$ so that $\Gamma_\tau \to \Lambda_\tau$ and the identity~\eqref{eq:yhatquery.randomcov.Minfinity} holds (conditionally on $\Lambda_\new)$.   Then, taking expectation over $\Lambda_\new$ in~\eqref{eq:yhatquery.randomcov.Minfinity}, we obtain
\begin{equation*}
    \E \left[\left.\widehat y_{\query}\right |
    x_{\query},w\right] \to x_{\query}^\top \cdot \left[\E \Lambda_\tau^2\right]  \left[\E \Lambda_\tau^3 \right]^{-1} \cdot \left[\E \Lambda_\tau \right] w.
\end{equation*}

If we consider the case $\lambda_{\tau, i} \iid \mathsf{Exponential}(1)$, so that $\E[\Lambda_\tau]=I_d$, $\E[\Lambda_\tau^2]= 2 I_d$, and $\E[\Lambda_\tau^3] = 6 I_d$, we get
\begin{equation*}
    \E \widehat y_{\query}  \to \frac{1}{3} \sip{w}{x_\query}.
\end{equation*}
This shows that for transformers with a single linear self-attention layer, training on in-context examples with random covariate distributions does not allow for in-context learning of a hypothesis class with varying covariate distributions.

\paragraph{Experiments with large, nonlinear transformers.} 
We have shown that even when trained on prompts with random covariance matrices, transformers with a single linear self-attention layer fail to in-context learn linear models with random covariance matrices.  We now investigate the behavior of more complex transformer architectures that are trained on in-context examples of linear models, both in the fixed-covariance case and in the random-covariance case.

We examine the performance of transformers with a GPT2 architecture~\citep{radford2019language} that are trained on linear regression tasks with mean-zero Gaussian features with either a fixed covariance matrix or random covariance matrices.  For the fixed covariance case, the covariance matrix is fixed to the identity matrix across prompts. For the random covariance case, covariates are drawn from $x\sim \normal(0,c\Lambda)$ where $\Lambda$ is diagonal with $\lambda_{i}\iid \mathsf{Exponential}(1)$ and $c> 0$ is a scaling factor. We set $c = 1$ during training and vary this value at test time. The transformer is trained using the procedure of~\citet{garg2022can} (see Appendix~\ref{appendix_experiments} for more details).  We consider linear models in $d=20$ dimensions and we train on prompt lengths of $N = 40, 70, 100$ with either fixed or random covariance matrices.  The performance of these trained models, when tested on new data with fixed covariance or random covariance matrices ($c = 1, 4, 9$), is represented in six curves in Figure~\ref{figure}. Using the calculation~\eqref{eq:yhatquery.randomcov.Minfinity}, we can compare the prediction error for the linear self-attention networks in the $M\to \infty$, $N\to \infty$ limit (the black dash line) to those of GPT2 architectures.  We additionally compare these models to the ordinary least-squares solution which is optimal for this task.  

\begin{figure}[t!]
    \centering
    \includegraphics[width=1.0 \textwidth]{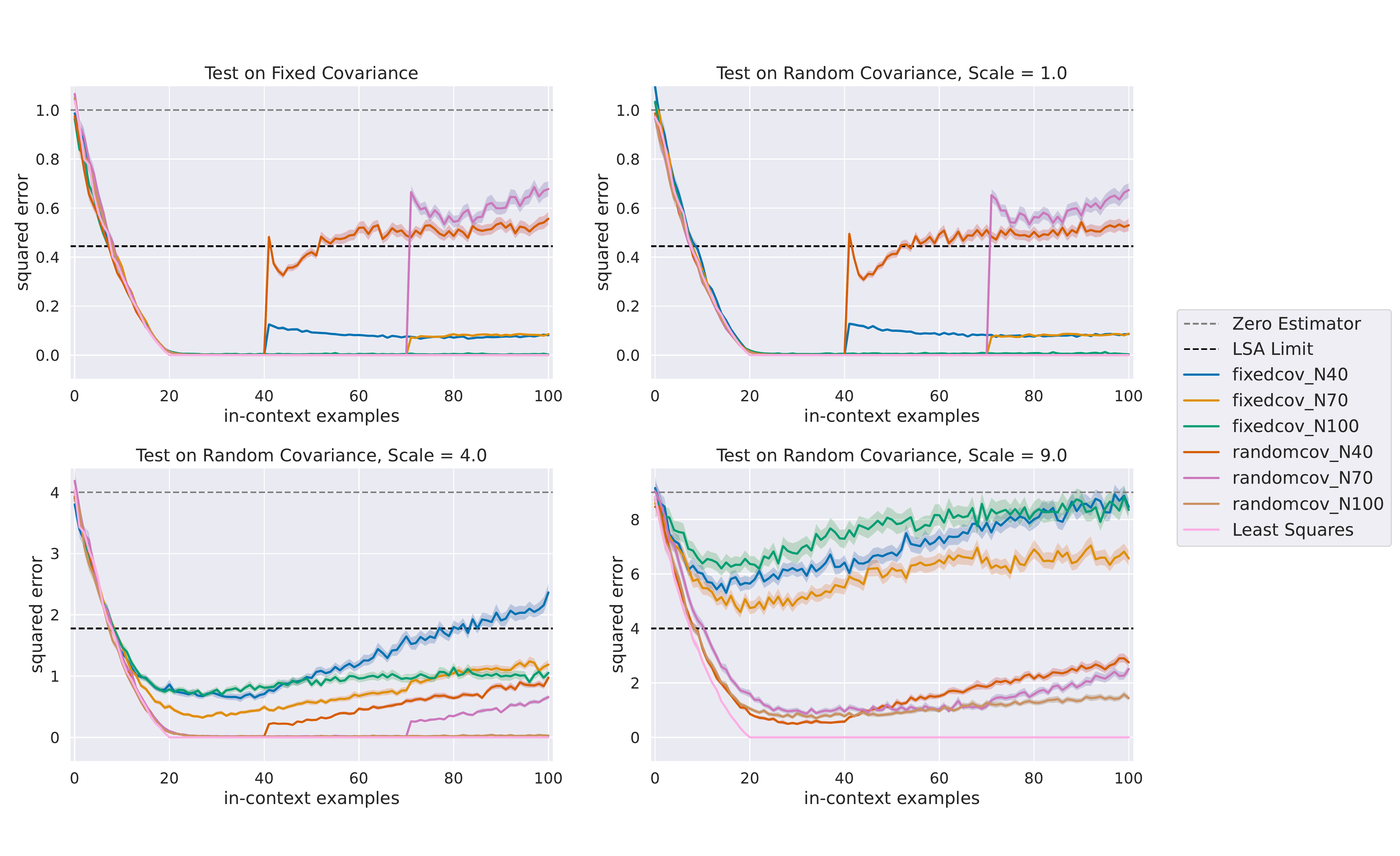}
    \caption{Normalized prediction error for transformers with GPT2 architectures as a function of the number of in-context test examples $M$ when trained on in-context examples of linear models in $d=20$ dimensions.  Colored lines correspond to different training context lengths $(N\in \{40,70,100\})$ and different training procedures (either a fixed identity covariance matrix or random diagonal covariance matrices with each diagonal element sampled i.i.d. from the standard exponential distribution).  The four figures correspond to evaluating on either fixed covariance or random covariance matrices of different scales.  The gray dashed line shows the prediction error of zero estimator and the black dashed line the prediction error of LSA model when $M,N \to \infty.$  The GPT2 models achieve smaller error when they are trained on random covariance matrices with larger contexts, but their prediction error spikes when evaluated on contexts larger than those they were trained on. }
    \label{figure}
\end{figure}

From the figure, we can see that the GPT2 model trained on fixed covariance succeeds in the random covariance setting if the variance is not too large, which shows that the larger nonlinear model is able to generalize better than the model with a single linear self-attention layer.  However, when the variance is large ($c = 4, 9$ for the bottom two figures), the GPT2 model trained with fixed covariance is unsuccessful.  When trained on random covariance, the model performs better for test prompts from higher-variance random covariance matrices, but still fails to match least squares when the scaling is largest ($c=9$). 

Furthermore, we notice some surprising behaviors when the test prompt length exceeds the training prompt length (i.e., $M > N$): there is an evident spike in prediction error, regardless of whether training and testing were performed on fixed or random covariance, and the spike appears to decrease when evaluated on prompts with higher variance.  Although we are unsure of why the spike should decrease with higher-variance prompts, the failure of large language models to generalize to larger contexts than they were trained on is a well-known problem~\citep{dai2019transformerxl,anil2022lengthgeneralization}.  In our setting, we conjecture that this spike in error comes from the absolute positional encodings in the GPT2 architecture.  The positional encodings are randomly-initialized and are learnable parameters but the encoding for position $i$ is only updated if the transformer encounters a prompt which has a context of length $i$.  Thus, when evaluating on prompts of length $M>N$, the model is relying upon random positional encodings for $M-N$ samples.  We note that a concurrent work has explored the performance of transformers with GPT2 architectures for in-context learning of linear models and found that removing positional encoders improves performance when evaluating on larger contexts~\citep{ahuja2023context}.  We leave further investigation of this behavior for future work.

\section{Proof ideas}
\label{sec:proof}
In this section, we briefly outline the proof sketch of Theorem \ref{thm:mainresult}.  The full proof of this theorem is left for Appendix~\ref{appendix:proof_main}.  
\subsection{Equivalence to a quadratic optimization problem}
We recall each task $\tau$ corresponds to a weight vector $w_\tau \sim \normal(0,I_d).$ The prompt inputs for this task are $x_{\tau,j} \iid \normal(0,\Lambda),$ which are also independent of $w_\tau$. The corresponding labels are $y_{\tau,j} = \sip{w_{\tau}}{x_{\tau,j}}.$ For each task $\tau,$ we can form the prompt into a token matrix $E_\tau \in \mathbb{R}^{(d+1) \times (N+1)}$ as in \eqref{eq:embedding.matrix.prompt}, with the right-bottom entry being zero.

The first key step in our proof is to recognize that the prediction $\widehat y_{\query}(E_\tau;\params)$ in the linear self-attention model can be written as the output of a quadratic function $u^\top H_\tau u$ for some matrix $H_\tau$ depending on the token embedding matrix $E_\tau$ and for some vector $u$ depending on $\params = (\WKQ, \WPV)$.  This is shown in the following lemma, the proof of which is provided in Appendix~\ref{sec_pf_quadratic}. 

\begin{restatable}{lemma}{lsaisquadratic}\label{lemma:lsa.is.quadratic}
Let $E_\tau \in \R^{(d+1)\times (N+1)}$ be an embedding matrix corresponding to a prompt of length $N$ and weight $w_\tau$.  Then the prediction $\widehat y_{\query}(E_\tau;\params)$ for the query covariate can be written as the output of a quadratic function,
    \begin{equation*}
        \widehat y_{\query}(E_\tau; \params) = \param^\top H_\tau \param,
    \end{equation*}
    where the matrix $H_\tau$ is defined as,
    \begin{equation}\label{eqn:def_H_tau}
        H_\tau = \frac{1}{2} X_\tau \otimes \left(\frac{E_\tau E_\tau^\top}{N}\right) \in \mathbb{R}^{(d+1)^2 \times (d+1)^2},\quad
        X_\tau = \begin{pmatrix}[1.5]
        0_{d \times d} & x_{\tau,\query}\\
        \left(x_{\tau,\query}\right)^\top & 0
        \end{pmatrix} \in \mathbb{R}^{(d+1) \times (d+1)}
    \end{equation}
    and
    \begin{equation*}
        \param = \vector(\bU) \in \mathbb{R}^{(d+1)^2}, \quad
        \bU = \begin{pmatrix}[1.5]
                \bUtl & \bUtr \\
                (\bUbl)^\top & \bUlast
        \end{pmatrix} \in \mathbb{R}^{(d+1) \times (d+1)},
    \end{equation*}
    where $\bUtl = \WKQ_{11} \in \mathbb{R}^{d \times d}, \bUtr = w_{21}^{PV} \in \mathbb{R}^{d \times 1},\bUbl = w_{21}^{KQ} \in \mathbb{R}^{d \times 1},\bUlast = w_{22}^{PV} \in \mathbb{R}$ correspond to particular components of 
    $\WPV$ and $\WKQ$, defined in \eqref{eqn:block_matrix}. 
\end{restatable}

\

This implies that we can write the original loss function \eqref{eqn:emp_loss_maintext} as
\begin{equation}\label{eqn:loss2_maintext}
    \widehat{L} = \frac{1}{2B}\sum_{\tau = 1}^B \left(u^\top H_\tau u - w_\tau^\top x_{\tau,\query}\right)^2.
\end{equation}

Thus, our problem is reduced to \textit{understanding the dynamics of an optimization algorithm defined in terms of a quadratic function}. We also note that this quadratic optimization problem is an instance of a rank-one matrix factorization problem, a problem well-studied in the deep learning theory literature \citep{gunasekar2017implicit,arora2019implicit,li2018algorithmic,chi2019nonconvex,belabbas2020implicit,li2020towards,jin2023understanding,soltanolkotabi2023implicit}.

 Note, however, this quadratic function is non-convex. To see this, we will show that $H_\tau$ has negative eigenvalues.  By standard properties of the Kronecker product, the eigenvalues of $H_\tau = \f 1 2 X_\tau \otimes \l(\f{E_\tau E_\tau^\top}N\r)$ are the products of the eigenvalues of $\f12 X_\tau$ and the eigenvalues of  $\f{E_\tau E_\tau^\top}N$.  
Since $E_\tau E_\tau^\top$ is symmetric and positive semi-definite, all of its eigenvalues are nonnegative.  Since $E_\tau E_\tau^\top$ is nonzero almost surely, it thus has at least one strictly positive eigenvalue.  Thus, if $X_\tau$ has any negative eigenvalues, $H_\tau$ does as well.  The characteristic polynomial of $X_\tau$ is given by,
\begin{align*}
    \mathrm{det}(\mu I - X_\tau) = \operatorname{det}\begin{pmatrix}[1.5]
        \mu I_{d} & -x_{\tau,\query} \\
        -x_{\tau,\query}^\top & \mu
    \end{pmatrix}
    = \mu^{d-1} \left(\mu^2 - \left\|x_{\tau,\query}\right\|_2^2\right).
\end{align*}
Therefore, we know almost surely, $X_\tau$ has one negative eigenvalue.   Thus $H_\tau$ has at least $d+1$ negative eigenvalues, and hence the quadratic form $u^\top H_\tau u$ is non-convex.

\subsection{Dynamical system of gradient flow}
We now describe the dynamical system for the coordinates of $\param$ above. We prove the following lemma in Appendix \ref{sec_pf_dyn}.

\begin{restatable}{lemma}{dynamicsforu}\label{lem:dyn_finite_data_main}
Let $\param = \vecU := \vector \begin{pmatrix}[1.5]
                \bUtl & \bUtr \\
                (\bUbl)^\top & \bUlast
        \end{pmatrix}$ as in Lemma~\ref{lemma:lsa.is.quadratic}.  Consider gradient flow over
    \begin{equation}\label{eqn:def_original_loss}
        L := \frac{1}{2} \E \left(\param^\top H_\tau \param - w_\tau^\top x_{\tau,\query}\right)^2
    \end{equation}
    with respect to $\param$ starting from an initial value satisfying Assumption~\ref{assume_init}. 
    Then the dynamics of $\bU$ follows
    \begin{equation}\label{eq:dyn_finite_data}
    \begin{aligned}
        \frac{\mathrm{d}}{\mathrm{d} t} \bUtl(t)
        &= - \bUlast^2 \Gamma \Lambda \bUtl \Lambda + \bUlast \Lambda^2\\
        \frac{\mathrm{d}}{\mathrm{d} t} \bUlast(t) &= - \operatorname{tr}\left[\bUlast \Gamma \Lambda \bUtl \Lambda (\bUtl)^\top - \Lambda^2 (\bUtl)^\top\right],
    \end{aligned}
    \end{equation}
    and $\bUtr (t)= 0_{d}, \bUbl (t)= 0_{d}$ for all $t \geq 0,$ where $   \Gamma = \left(1 + \frac{1}{N}\right)\Lambda + \frac{1}{N}\operatorname{tr}(\Lambda) I_d \in \mathbb{R}^{d \times d}$.
\end{restatable}

\ 

We see that the dynamics are governed by a complex system of $d^2+1$ coupled differential equations. Moreover, basic calculus (for details, see Lemma~\ref{lemma:loss}) shows that these dynamics are the same as those of gradient flow on the following objective function:
\begin{equation} \label{eq:equivalent.objective.population}
\tilde \ell : \R^{d\times d} \times \R \to \R,\quad \tilde \ell\left(\bUtl,\bUlast\right) =  \operatorname{tr}\left[\frac{1}{2} \bUlast^2 \Gamma \Lambda \bUtl \Lambda (\bUtl)^\top- \bUlast \Lambda^2 (\bUtl)^\top\right].
\end{equation}
Actually, the loss function $\tilde \ell$ is simply the loss function $L$ in \eqref{eqn:def_original_loss} plus some constants that do not depend on the parameter $u$. Therefore our problem is reduced to studying the dynamics of gradient flow on the above objective function.

Our next key observation is that the set of global minima for $\tilde \ell$ satisfies the condition $\bUlast \bUtl = \Gamma^{-1}$.  Thus, if we can establish global convergence of gradient flow over the above objective function $\tilde \ell$, then we have that $\bUlast(t) \bUtl(t) \to \Gamma^{-1} \approx_{N\to \infty} \Lambda^{-1}$.

\begin{restatable}{lemma}{globalminproperty}\label{lemma:globalmin}
For any global minimum of $\tilde \ell$, we have 
\begin{equation}
    \bUlast \bUtl = \Gamma^{-1}.
\end{equation}
\end{restatable}

Putting this together with Lemma~\ref{lem:dyn_finite_data_main}, we see that at those global minima of the population objective satisfying $\bUtl = (c\Gamma)^{-1}$, $\bUlast =c$ and $\bUtr = \bUbl = 0_d$, the transformer's predictions for a new linear regression task prompt are given by
\[ \widehat y_{\query}(E;\params) =  \f 1 M \summ i M y_i x_i^\top  \Gamma^{-1} x_{\query} =w^\top \l( \f 1 M \summ i M x_i x_i^\top \r) \Gamma^{-1} x_{\query} \approx w^\top x_{\query}.\]
Thus, the only remaining task is to show global convergence when gradient flow has an initialization satisfying Assumption~\ref{assume_init}. 

\subsection{PL inequality and global convergence}
We now show that although the optimization problem is non-convex, a Polyak-\L ojasiewicz (PL) inequality holds, which implies that gradient flow converges to a global minimum.  Moreover, we can exactly calculate the limiting value of $\bUtl$ and $\bUlast$.

\begin{restatable}{lemma}{plinequality}\label{lemma:proxypl}
Suppose the initialization of gradient flow satisfies Assumption~\ref{assume_init} with initialization scale satisfying $\sigma^2 < \frac{2}{\sqrt{d} \left\|\Gamma\right\|_{op}}$ for $\Gamma = (1 + \frac 1N) \Lambda + \f{\tr(\Lambda)}N I_d$. 
If we define
\begin{equation}\label{eqn:def_mu_maintext}
    \mu := \frac{\sigma^2}{\sqrt{d}\left\|\Lambda\right\|_{op}^2 \operatorname{tr}\left(\Gamma^{-1} \Lambda^{-1}\right)\operatorname{tr}\left(\Lambda^{-1}\right)} \left\|\Lambda \Theta\right\|_F^2 \left[2 - \sqrt{d}\sigma^2 \left\| \Gamma\right\|_{op}\right] > 0,
\end{equation}
then gradient flow on $\tilde \ell$ with respect to $\bUtl$ and $\bUlast$ satisfies, for any $t \geq 0,$
\begin{equation}
    \left\|\nabla \tilde \ell(\bUtl(t), \bUlast(t))\right\|_2^2 := \left\|\frac{\partial \tilde \ell}{\partial \bUtl}\right\|_F^2 + \left|\frac{\partial \tilde \ell}{\partial \bUlast}\right|^2 \geq \mu \left(\tilde \ell(\bUtl(t), \bUlast(t)) - \min_{\bUtl \in \mathbb{R}^{d \times d}, \bUlast \in \mathbb{R}} \tilde \ell(\bUtl,\bUlast)\right).
\end{equation}
Moreover, gradient flow converges to the global minimum of $\tilde \ell$, and $\bUtl$ and $\bUlast$ converge to the following,
\begin{equation}
    \lim_{t \to \infty} \bUlast(t) = \left\|\Gamma^{-1}\right\|_F^{\frac{1}{2}} \text{ and }
    \lim_{t \to \infty} \bUtl(t) = \left\|\Gamma^{-1}\right\|_F^{-\frac{1}{2}} \Gamma^{-1}.
\end{equation}
\end{restatable}

\

With these observations, proving Theorem~\ref{thm:mainresult} becomes a direct application of Lemma \ref{lemma:lsa.is.quadratic}, \ref{lem:dyn_finite_data_main}, \ref{lemma:globalmin}, and Lemma \ref{lemma:proxypl}. It then only requires translating $\bUtl$ and $\bUlast$ back to the original parameterization using $\WPV$ and $\WKQ$.

\section{Conclusion and future work}

In this work, we investigated the dynamics of in-context learning of transformers with a single linear self-attention layer under gradient flow on the population loss.  In particular, we analyzed the dynamics of these transformers when trained on prompts consisting of random instances of noiseless linear models over anisotropic Gaussian marginals.  We showed that despite non-convexity, gradient flow from a suitable random initialization converges to a global minimum of the population objective.  We characterized the prediction error of the trained transformer when given a new prompt that consists of a training dataset where the responses are a nonlinear function of the inputs.  We showed how the trained transformer is naturally robust to shifts in the task and query distributions but is brittle to distribution shifts between the covariates seen during training and the covariates seen at test time, matching the empirical observations on trained transformer models of~\citet{garg2022can}.

There are a number of natural directions for future research.  First, our results hold for gradient flow on the population loss with a particular class of random initialization schemes.  It is a natural question if similar results would hold for stochastic gradient descent with finite step sizes and for more general initializations.  Further, we restricted our attention to transformers with a single linear self-attention layer.  Although this model class is rich enough to allow for in-context learning of linear predictors, we are particularly interested in understanding the dynamics of in-context learning in nonlinear and deep transformers.  

Finally, the framework of in-context learning introduced in prior work was restricted to the setting where the marginal distribution over the covariates $(\calD_x)$ was fixed across prompts.  This allows for guarantees akin to distribution-specific PAC learning, where the trained transformer is able to achieve small prediction error when given a test prompt consisting of linear regression data when the marginals over the covariates are fixed.  However, other learning algorithms (such as ordinary least squares) are able to achieve small prediction error for prompts corresponding to well-specified linear regression tasks for very general classes of distributions over the covariates.  As we showed in Section~\ref{sec:main.randomcov}, when transformers with a single linear self-attention layer are trained on prompts where the covariate distributions are themselves sampled from a distribution, they do not succeed on test prompts with covariate distributions sampled from the same distribution. By contrast, we demonstrated with experiments that larger, nonlinear transformer architectures appear to be more successful in this setting but are still sub-optimal.  Developing a better understanding of the dynamics of in-context learning when the covariate distribution varies across prompts is an intriguing direction for future research.

\subsection*{Acknowledgements}

We gratefully acknowledge the support of the NSF and the Simons Foundation for the Collaboration on the Theoretical Foundations of Deep Learning through awards DMS-2031883 and \#814639, and of the NSF through grant DMS-2023505.

\newpage 
{\hypersetup{linkcolor=Black}\tableofcontents}
\appendix

\newpage
\section{Proof of Theorem \ref{thm:mainresult}}\label{appendix:proof_main}
In this section, we prove Lemma \ref{lemma:lsa.is.quadratic}, Lemma \ref{lem:dyn_finite_data_main}, Lemma \ref{lemma:globalmin} and Lemma \ref{lemma:proxypl}. Theorem \ref{thm:mainresult} is a natural corollary of these four lemmas when we translate $\bUlast$ and $\bUtl$ back to $\WPV$ and $\WKQ.$ 

\subsection{Proof of Lemma \ref{lemma:lsa.is.quadratic}}\label{sec_pf_quadratic}
For the reader's convenience, we restate the lemma below. 

\lsaisquadratic*
\begin{proof}
First, we decompose $W_{PV}$ and $W_{KQ}$ in the way above. From the definition, we know $\widehat{y}_{\tau,\query}$ is the right-bottom entry of $f_{\lsa}(E_\tau),$ which is
\begin{equation*}
    \widehat{y}_{\tau,\query}
    = \begin{pmatrix}[1.5]
    (\bUtr)^\top & \bUlast
    \end{pmatrix} \left(\frac{E_\tau E_\tau^\top}{N}\right)
    \begin{pmatrix}[1.5]
    \bUtl \\
    (\bUbl)^\top
    \end{pmatrix}
    x_{\tau,\query}.
\end{equation*}
We denote $u_i \in \mathbb{R}^{d+1}$ as the $i$-th column of $\begin{pmatrix}[0.5]
    \bUtl \\
    (\bUbl)^\top
    \end{pmatrix}$ and $x_{\tau,\query}^i$ as the $i$-th entry of $x_{\tau,\query}$ for $i \in [d].$ Then, we have
\begin{align*}
    \widehat{y}_{\tau,\query}
    &= \sum_{i=1}^d x_{\tau,\query}^i \begin{pmatrix}[1.5]
    (\bUtr)^\top & \bUlast
    \end{pmatrix} \left(\frac{E_\tau E_\tau^\top}{N}\right) u_i
    = \sum_{i=1}^d \operatorname{tr}\left[u_i \begin{pmatrix}[1.5]
    (\bUtr)^\top & \bUlast
    \end{pmatrix} \cdot x_{\tau,\query}^i \left(\frac{E_\tau E_\tau^\top}{N}\right)\right] \\
    &= \operatorname{tr}\left[\vector\left[\begin{pmatrix}[1.5]
    \bUtl \\
    (\bUbl)^\top
    \end{pmatrix}\right] \begin{pmatrix}[1.5]
    (\bUtr)^\top & \bUlast
    \end{pmatrix} \cdot x_{\tau,\query}^\top \otimes \left(\frac{E_\tau E_\tau^\top}{N}\right)\right] \\
    &= \frac{1}{2}\operatorname{tr}\left[\vector\left[\begin{pmatrix}[1.5]
    \bUtl & \bUtr \\
    (\bUbl)^\top & \bUlast
    \end{pmatrix}\right] \vector^\top\left[\begin{pmatrix}[1.5]
    \bUtl & \bUtr \\
    (\bUbl)^\top & \bUlast
    \end{pmatrix}\right] \cdot 
    \begin{pmatrix}
    0_{d(d+1) \times d(d+1)} & x_{\tau,\query} \otimes \left(\frac{E_\tau E_\tau^\top}{N}\right) \\
    x_{\tau,\query}^\top \otimes \left(\frac{E_\tau E_\tau^\top}{N}\right) & 0_{(d+1) \times (d+1)}
    \end{pmatrix}\right] \\
    &= \frac{1}{2} \operatorname{tr}\left[\param \param^\top \cdot X_\tau \otimes \left(\frac{E_\tau E_\tau^\top}{N}\right)\right] \\
    &= \left\langle H_\tau, \param \param^\top \right\rangle.
\end{align*}
Here, we use some algebraic facts about matrix vectorization, Kronecker product and trace. For reference, we refer to \citep{petersen2008matrix}. 
\end{proof}

\subsection{Proof of Lemma \ref{lem:dyn_finite_data_main}}\label{sec_pf_dyn}
For the reader's convenience, we restate the lemma below. 
\dynamicsforu*
\begin{proof}
    From the definition of $L$ in \eqref{eqn:def_original_loss} and the dynamics of gradient flow, we calculate the derivatives of $\param$. Here, we use the chain rule and some facts about matrix derivatives. See Lemma \ref{lem:matrix} for reference.
    \begin{equation}\label{eqn:gradient_original}
        \frac{\mathrm{d} \param}{\mathrm{d} t} = - 2 \E \left(\langle H_\tau, \param \param^\top \rangle H_\tau \right) \param + 2 \E \left(w_\tau^\top x_{\tau,\query} H_\tau\right)\param.
    \end{equation}

\paragraph{Step One: Calculate the Second Term}
    
    We first calculate the second term. From the definition of $H_\tau,$ we have
\begin{align*}
    \E \left[w_\tau^\top x_{\tau,\query} H_\tau\right] 
    &= \frac{1}{2}\sum_{i=1}^d \E \left[\left(x_{\tau,\query}^i X_\tau\right) \otimes \left(w_\tau^i \frac{E_\tau E_\tau^\top}{N}\right) \right].
\end{align*}

For ease of notation, we denote
\begin{equation}\label{eqn:def_hlambda}
    \hLambda_\tau := \frac{1}{N} \sum_{i=1}^N x_{\tau,i} x_{\tau,i}^\top.
\end{equation}
Then, from the definition of $\frac{E_\tau E_\tau^\top}{N},$ we know
\begin{equation*}
    \frac{E_\tau E_\tau^\top}{N}
    = \begin{pmatrix}[1.5]
          \hLambda_\tau + \frac{1}{N} x_{\tau,\query} \cdot x_{\tau,\query}^\top &  \hLambda_\tau w_\tau \\
          w_\tau \hLambda_\tau & w_\tau^\top \hLambda_\tau w_\tau
    \end{pmatrix}.
\end{equation*}
Since $w_\tau \sim \normal(0,I_d)$ is independent of all prompt inputs and query input, 
we have
\begin{align*}
    &\frac{1}{2}\sum_{i=1}^d \E \left[\left(x_{\tau,\query}^i X_\tau\right) \otimes \left(\frac{w_\tau^i}{N} \begin{pmatrix}
          x_{\tau,\query} \cdot x_{\tau,\query}^\top &  0 \\
            0 & 0
    \end{pmatrix}\right) \right] \\
    =& \frac{1}{2}\sum_{i=1}^d \E \left[ \E \left[\left(x_{\tau,\query}^i X_\tau\right) \otimes \left(\frac{w_\tau^i}{N} \begin{pmatrix}
          x_{\tau,\query} \cdot x_{\tau,\query}^\top &  0 \\
            0 & 0
    \end{pmatrix}\right) \right] \bigg| x_{\tau,\query}\right] \\
    =& \frac{1}{2}\sum_{i=1}^d \E \left[\left(x_{\tau,\query}^i X_\tau\right) \otimes \left(\frac{\E \left[w_\tau^i \mid x_{\tau,\query}\right]}{N} \begin{pmatrix}
          x_{\tau,\query} \cdot x_{\tau,\query}^\top &  0 \\
            0 & 0
    \end{pmatrix}\right) \right] = 0.
\end{align*}
Therefore, we have
\begin{equation*}
    \E \left[w_\tau^\top x_{\tau,\query} H_\tau\right] 
    = \frac{1}{2}\sum_{i=1}^d \E \left[\left(x_{\tau,\query}^i X_\tau\right) \otimes \left(w_\tau^i \begin{pmatrix}[1.5]
          \hLambda_\tau &  \hLambda_\tau w_\tau \\
          w_\tau^\top \hLambda_\tau &
          w_\tau^\top \hLambda_\tau w_\tau.
    \end{pmatrix}\right)\right].
\end{equation*}
Since $X_\tau$ only depends on $x_{\tau,\query}$ by definition, and $x_{\tau,\query}$ is independent of $w_{\tau}$ and $x_{\tau,i}, i = 1,2,...,N,$ we have 
\begin{align*}
    \E \left[w_\tau^\top x_{\tau,\query} H_\tau\right] 
    &= \frac{1}{2}\sum_{i=1}^d \left[\E\left(x_{\tau,\query}^i X_\tau\right) \otimes \E \left(w_\tau^i \begin{pmatrix}[1.5]
          \hLambda_\tau &  \hLambda_\tau w_\tau \\
          w_\tau^\top \hLambda_\tau &
          w_\tau^\top \hLambda_\tau w_\tau.
    \end{pmatrix}\right)\right] \\
    &= \frac{1}{2}\sum_{i=1}^d \left[
    \begin{pmatrix}[1.5]
            0_{d\times d} & \Lambda_i \\
            \Lambda_i^\top & 0
    \end{pmatrix} \otimes
    \begin{pmatrix}[1.5]
          \E (w_\tau^i) \Lambda &  \Lambda \E (w_\tau^i w_\tau)\\
          \E (w_\tau^i w_\tau^\top) \Lambda &
          \E \left(w_\tau^i w_\tau^\top \Lambda w_\tau\right)
    \end{pmatrix}\right] \\
    &= \frac{1}{2}\sum_{i=1}^d
    \begin{pmatrix}[1.5]
            0_{d\times d} & \Lambda_i \\
            \Lambda_i^\top & 0
    \end{pmatrix} \otimes \begin{pmatrix}[1.5]
            0_{d\times d} & \Lambda_i \\
            \Lambda_i^\top & 0
    \end{pmatrix},
\end{align*}
where $\Lambda_i$ denotes $\Lambda_{:i}$. Here, the second line comes from the fact that $\E \hLambda_\tau = \Lambda$, and that $w_\tau$ is independent of all prompt input and query input. The last line comes from the fact that $w_\tau \sim \normal(0,I_d).$ Therefore, simple computation shows that
\begin{equation}\label{eqn:gradient_second_term}
    \E \left[w_\tau^\top x_{\tau,\query} H_\tau\right] \param  = \frac{1}{2} \begin{pmatrix}[1.5]
    \boldsymbol{0}_{d(d+1) \times d(d+1)} & A \\
    A^\top & \boldsymbol{0}_{(d+1) \times (d+1)}
    \end{pmatrix} \cdot \param,
\end{equation}
where
\begin{equation}\label{eqn:def_A_V}
    A =
    \begin{pmatrix}[1.5]
            V_{1} + V_1^\top \\ V_{2} + V_2^\top \\ ... \\ V_{d} + V_d^\top
        \end{pmatrix} \in \mathbb{R}^{d(d+1) \times (d+1)}, \quad
    V_{j} =  \begin{pmatrix}[1.5]
            0_{d \times d} & \sum_{i=1}^d \Lambda_{ij} \Lambda_i \\
            0 & 0
        \end{pmatrix}
        =  \begin{pmatrix}[1.5]
            0_{d \times d} & \Lambda \Lambda_j \\
            0 & 0
        \end{pmatrix}
        \in \mathbb{R}^{(d+1) \times (d+1)}.
\end{equation}
    
\paragraph{Step Two: Calculate the First Term}

Next, we compute the first term in \eqref{eqn:gradient_original}, namely 
$$
D := 2 \E \left(\langle H_\tau, \param \param ^\top \rangle H_\tau \param \right).
$$
For simplicity, we denote $Z_\tau := \frac{1}{N} E_\tau E_\tau^\top.$ Using the definition of $H_\tau$ in \eqref{eqn:def_H_tau} and Lemma \ref{lem:matrix}, we have
\begin{align*}
    D &= 2 \E \left(\langle H_\tau, \param \param^\top \rangle H_\tau \param \right) \tag{definition}\\
    &= \frac{1}{2} \E \left[\operatorname{tr}\left(X_\tau \otimes Z_\tau \vecU \vecU^\top \right)\left(X_\tau \otimes Z_\tau \right) \vecU\right] \tag{definition of $H_\tau$ in \eqref{eqn:def_H_tau} and $\param = \vector(\bU)$}\\
    &= \frac{1}{2} \E \left[\operatorname{tr}\left(\vecZUX \vecU^\top \right) \vecZUX\right] \hspace{40ex} \tag{$\vector(AXB) = (B^\top \otimes A) \vector(X)$ in Lemma \ref{lem:matrix}}\\
    &= \frac{1}{2} \E \left[\vecU^\top \cdot \vecZUX \cdot \vecZUX\right] \hspace{40ex} \tag{property of trace operator} \\
    &= \frac{1}{2} \E \left[\sum_{i,j=1}^{d+1} \left(\ZUX_{ij} \bU_{ij} \right) \vecZUX\right].
\end{align*}

\paragraph{Step Three: $\bUtr$ and $\bUbl$ Vanish}

We first prove that if $\bUtr = \bUbl = 0_{d}$, then $\frac{\mathrm{d}}{\mathrm{d}t} \bUtr = 0_d$ and $\frac{\mathrm{d}}{\mathrm{d}t} \bUbl = 0_d.$ If this is true, then these two blocks will be zero all the time since we assume they are zero at initial time in Assumption~\ref{assume_init}. We denote $A_{k:}$ and $A_{:k}$ as the k-th row and k-th column of matrix $A,$ respectively. 

Under the assumption that $\bUtr = \bUbl = 0_d$, we first compute
\begin{equation*}
    \ZUX = 
    \begin{pmatrix}[1.5]
        \hLambda_\tau w_\tau \bUlast x_{\tau,\query}^\top & \left(\hLambda_\tau + \frac{1}{N} x_{\tau,\query} \cdot x_{\tau,\query}^\top\right) \bUtl x_{\tau,\query}\\
        w_\tau^\top \left(\hLambda_\tau\right) w_\tau \bUlast x_{\tau,\query}^\top & w_\tau^\top \left(\hLambda_\tau\right) \bUtl x_{\tau,\query}
    \end{pmatrix}.
\end{equation*}
Written in an entry-wise manner, it will be
\begin{equation}\label{eqn:ZUX}
    \ZUX_{kl} = \begin{cases}
    \left(\hLambda_\tau\right)_{k:} w_\tau \bUlast x_{\tau,\query}^l & k,l \in [d] \\
    \left(\hLambda_\tau + \frac{1}{N} x_{\tau,\query} \cdot x_{\tau,\query}^\top\right)_{k:} \bUtl x_{\tau,\query} \hspace{10ex} & k \in [d], l = d+1 \\
    w_\tau^\top \left(\hLambda_\tau\right) w_\tau \bUlast x_{\tau,\query}^l & l \in [d], k = d+1 \\
    w_\tau^\top \left(\hLambda_\tau\right) \bUtl x_{\tau,\query} & k = l = d+1
    \end{cases}.
\end{equation}

We use $D_{ij}$ to denote the $(i,j)$-th entry of the $(d+1)\times(d+1)$ matrix $\bar{D}$ such that $\vector(\bar{D}) = D$. Now we fix a $k \in [d],$ then
\begin{align}\label{eqn:D_k_d+1}
    D_{k,d+1} 
    &= \frac{1}{2} \E \left[\sum_{i,j=1}^{d+1} \left(\ZUX_{ij} \bU_{ij} \right) \ZUX_{k,d+1}\right] \notag\\
    &= \frac{1}{2} \E \left[\sum_{i,j=1}^{d} \left(\ZUX_{ij} \bU_{ij} \right) \ZUX_{k,d+1}\right] + \frac{1}{2} \E \left[\left(\ZUX_{d+1,d+1} \bUlast \right) \ZUX_{k,d+1}\right],
\end{align}
since $\bU_{i,d+1} = \bU_{d+1,i} = 0$ for any $i \in [d].$ For the first term in the right hand side of last equation, we fix $i,j \in [d]$ and have
\begin{align*}
    &\E \left(\ZUX_{ij} \bU_{ij} \right) \ZUX_{k,d+1} \\
    =& \E \left(\bU_{ij} \left(\hLambda_\tau\right)_{i:} w_\tau \bUlast x_{\tau,\query}^j \cdot \left(\hLambda_\tau + \frac{1}{N} x_{\tau,\query}\cdot x_{\tau,\query}^\top\right)_{k:} \bUtl x_{\tau,\query}\right) = 0,
\end{align*}
since $w_\tau$ is independent with all prompt input and query input, namely all $x_{\tau,i}$ for $i \in [\query],$ and $w_\tau$ is mean zero. Similarly, for the second term of \eqref{eqn:D_k_d+1}, we have
\begin{align*}
    &\E \left(\ZUX_{d+1,d+1} \bUlast \right) \ZUX_{k,d+1} \\
    =& \E \left(\bUlast w_\tau^\top \left(\hLambda_\tau\right) \bUtl x_{\tau,\query} \cdot \left(\hLambda_\tau +\frac{1}{N} x_{\tau,\query} \cdot x_{\tau,\query}\right)_{k:} \bUtl x_{\tau,\query}\right) = 0
\end{align*}
since $\E\left(w_\tau^\top\right) = 0$ and $w_\tau$ is independent of all $x_{\tau,i}$ for $i \in [\query]$. Therefore, we have $D_{k,d+1} = 0$ for $k \in [d].$ Similar calculation shows that $D_{d+1,k} = 0$ for $k \in [d].$ 

\ 

For $k \in [d],$ to calculate the derivative of $\bU_{k,d+1},$ it suffices to further calculate the inner product of the $d(d+1) + k$ th row of $\E\left[w_\tau^\top x_{\tau,\query} H_\tau\right]$ and $\param.$ From \eqref{eqn:gradient_second_term}, we know this is
\begin{equation*}
    \frac{1}{2} \sum_{j=1}^d \Lambda_k^\top \Lambda_j \bU_{d+1,j} = 0
\end{equation*}
given that $\bUtr = \bUbl = 0_d.$ Therefore, we conclude that the derivative of $\bU_{k,d+1}$ will vanish given $\bUtr = \bUbl = 0_d.$ Similarly, we conclude the same result for $\bU_{d+1,k}$ for $k \in [d].$ Therefore, we know $\bUtr = 0_d$ and $\bUbl = 0_d$ for all time $t \geq 0.$
    
\paragraph{Step Four: Dynamics of $\bUtl$}
    
Next, we calculate the derivatives of $\bUtl$ given $\bUtr = \bUbl = 0_d$. For a fixed pair of $k,l \in [d],$ we have
\begin{align*}
    D_{kl} &= \frac{1}{2} \E \left[\sum_{i,j=1}^{d} \left(\ZUX_{ij} \bU_{ij} \right) \ZUX_{kl}\right] + \frac{1}{2} \E \left[\left(\ZUX_{d+1,d+1} \bUlast \right) \ZUX_{kl}\right].
\end{align*}
For fixed $i,j \in [d],$  we have
\begin{align*}
    \E \left[\left(\ZUX_{ij} \bU_{ij} \right) \ZUX_{kl} \right] 
    &= \bU_{ij} \bUlast^2 \E \left[\left(\hLambda_\tau\right)_{i:} w_{\tau} x_{\tau,\query}^j x_{\tau,\query}^l w_{\tau}^\top \left(\hLambda_\tau\right)_{:k} \right]\\
    &= \bU_{ij} \bUlast^2 \E \left[x_{\tau,\query}^j x_{\tau,\query}^l \right] \cdot \E \left[\left(\hLambda_\tau\right)_{i:} \left(\hLambda_\tau\right)_{:k} \right] \\
    &= \bU_{ij} \bUlast^2 \Lambda_{\tau,jl} \E \left[\left(\hLambda_\tau\right)_{i:} \left(\hLambda_\tau\right)_{:k} \right].
\end{align*}
Therefore, we sum over $i,j \in [d]$ to get 
    \begin{equation*}
        \frac{1}{2} \E \left[\sum_{i,j=1}^{d} \left(\ZUX_{ij} \bU_{ij} \right) \ZUX_{kl}\right] 
        = \frac{1}{2} \bUlast^2 \E \left(\left(\hLambda_\tau\right)_{k:} \left(\hLambda_\tau\right) \right) \bUtl \Lambda_{l}
    \end{equation*}
   For the last term, we have
    \begin{equation*}
        \frac{1}{2}\E \left[\left(\ZUX_{d+1,d+1} \bUlast \right) \ZUX_{kl}\right] = \frac{1}{2} \bUlast^2 \E \left(\left(\hLambda_\tau\right)_{k:} \left(\hLambda_\tau\right) \right) \bUtl \Lambda_{l}.
    \end{equation*}
    So we have
    \begin{equation*}
        D_{kl} = \bUlast^2 \E \left(\left(\hLambda_\tau\right)_{k:} \left(\hLambda_\tau\right) \right) \bUtl \Lambda_{l}.
    \end{equation*}
    Additionally, we have 
    \begin{align*}
        2\left[\E \left(w_\tau^\top x_{\tau,\query} H_\tau\right)\param\right]_{(l-1)(d+1) + k} 
        &= \left[\begin{pmatrix}[1.5]
    \boldsymbol{0}_{d(d+1) \times d(d+1)} & A \\
    A^\top & \boldsymbol{0}_{(d+1) \times (d+1)}
    \end{pmatrix} \cdot \param\right]_{(l-1)(d+1) + k} \tag{definition} \\
    &= \begin{pmatrix}[1.5]
        0_{(d+1)\times d(d+1)} & V_l + V_l^\top
    \end{pmatrix}_{k:} \cdot U \tag{definition of $A$ in \eqref{eqn:def_A_V}}\\
    &= \Lambda_k^\top \Lambda_l \bUlast. \tag{definition of $V_i$ in \eqref{eqn:def_A_V}}
    \end{align*}
    Therefore, we have that for $k,l \in [d],$ the dynamics of $\bU_{kl}$ is
    \begin{equation*}
        \frac{\mathrm{d}}{\mathrm{d} t} \bU_{kl}
        = - \bUlast^2 \E \left(\left(\hLambda_\tau\right)_{k:} \left(\hLambda_\tau\right) \right) \bUtl \Lambda_{l} + \bUlast \Lambda_k^\top \Lambda_l,
    \end{equation*}
    which implies
    \begin{equation*}
        \frac{\mathrm{d}}{\mathrm{d} t} \bUtl 
        = - \bUlast^2 \E \left(\left(\hLambda_\tau\right)^2 \right) \bUtl \Lambda + \bUlast \Lambda^2.
    \end{equation*}

\ 

From the definition of $\hLambda_\tau$ (equation \eqref{eqn:def_hlambda}), the independence and Gaussianity of $x_{\tau,i}$ and Lemma \ref{lem:fourth_moment}, we compute
\begin{align*}
    \E \left(\left(\hLambda_\tau\right)^2\right)
    &= \E \left(\left(\frac{1}{N}\sum_{i=1}^N x_{\tau,i} x_{\tau,i}^\top \right)^2\right) \hspace{30ex} \tag{definition \eqref{eqn:def_hlambda}}\\
    &= \frac{N-1}{N} \left[\E \left(x_{\tau,1} x_{\tau,1}^\top\right)\right]^2 + \frac{1}{N} \E \left(x_{\tau,1} x_{\tau,1}^\top x_{\tau,1} x_{\tau,1}^\top\right) \tag{independence between prompt input}\\
    &= \frac{N+1}{N} \Lambda^2 + \frac{1}{N}\operatorname{tr}(\Lambda)\Lambda. \tag{Lemma \ref{lem:fourth_moment}}
\end{align*}
We define 
\begin{equation}\label{eqn:def_Gamma}
    \Gamma := \frac{N+1}{N} \Lambda + \frac{1}{N}\operatorname{tr}(\Lambda)I_d.
\end{equation}
Then, from \eqref{eqn:gradient_original}, we know the dynamics of $\bUtl$ is 
\begin{equation}\label{eqn:dyn_U11}
    \frac{\mathrm{d}}{\mathrm{d} t} \bUtl 
    = - \bUlast^2 \Gamma \Lambda \bUtl \Lambda + \bUlast \Lambda^2.
\end{equation}
    
\paragraph{Step Five: Dynamics of $\bUlast$}

Finally, we compute the dynamics of $\bUlast.$ We have
\begin{align}\label{eqn:D_d+1_d+1}
    D_{d+1,d+1} &= \frac{1}{2} \E \left[\sum_{i,j=1}^{d} \left(\ZUX_{ij} \bU_{ij} \right) \ZUX_{d+1,d+1}\right] + \frac{1}{2} \E \left[\left(\ZUX_{d+1,d+1} \bUlast \right) \ZUX_{d+1,d+1}\right].
\end{align}
For the first term above, we have
\begin{align*}
    &\E \left[\sum_{i,j=1}^{d} \left(\ZUX_{ij} \bU_{ij} \right) \ZUX_{d+1,d+1}\right] \\
    =& \bUlast \sum_{i,j=1}^d \bU_{ij} \E \left[\left(\hLambda_{\tau}\right)_{i:} \cdot w_\tau w_\tau^\top \cdot \left(\hLambda_{\tau}\right) \cdot \bUtl x_{\tau,\query} x_{\tau,\query}^j\right] \tag{from \eqref{eqn:ZUX}}\\
    =& \bUlast \sum_{i,j=1}^d \bU_{ij} \E \left[\left(\hLambda_{\tau}\right)_{i:}  \cdot \left(\hLambda_{\tau}\right) \cdot \bUtl x_{\tau,\query} x_{\tau,\query}^j\right] \tag{independence and distribution of $w_\tau$} \\
    =& \bUlast \sum_{i,j=1}^d \bU_{ij} \E \left[\left(\hLambda_{\tau}\right)_{i:}  \cdot \left(\hLambda_{\tau}\right) \cdot \bUtl \Lambda_{j}\right] \tag{independence between prompt covariates}\\
    =& \bUlast \E \operatorname{tr} \left[\sum_{i,j=1}^d \Lambda_j \bU_{ij} \left(\hLambda_{\tau}\right)_{i:}  \cdot \left(\hLambda_{\tau}\right) \bUtl\right]
    = \bUlast \E \operatorname{tr} \left[\Lambda (\bUtl)^\top \left(\hLambda_{\tau}\right)^2 \bUtl\right] \\
    =& \bUlast \operatorname{tr} \left[\E\left(\hLambda_{\tau}\right)^2 \bUtl \Lambda (\bUtl)^\top \right].
\end{align*}
For the second term in \eqref{eqn:D_d+1_d+1}, we have
\begin{align*}
    \E \left[\left(\ZUX_{d+1,d+1} \bUlast \right) \ZUX_{d+1,d+1}\right] 
    &= \bUlast \E \left[w_\tau^\top \left(\hLambda_{\tau}\right) \bUtl x_{\tau,\query} x_{\tau,\query}^\top (\bUtl)^\top \left(\hLambda_{\tau}\right) w_\tau\right] \tag{from \eqref{eqn:ZUX}}\\
    &= \bUlast \E \operatorname{tr} \left[w_\tau w_\tau^\top \left(\hLambda_{\tau}\right) \bUtl x_{\tau,\query} x_{\tau,\query}^\top (\bUtl)^\top \left(\hLambda_{\tau}\right)\right] \\
    &= \bUlast \E \operatorname{tr} \left[\left(\hLambda_{\tau}\right) \bUtl \Lambda (\bUtl)^\top \left(\hLambda_{\tau}\right)\right] \\
    &= \bUlast \operatorname{\tr} \left[\E \left(\hLambda_{\tau}\right)^2 \bUtl \Lambda (\bUtl)^\top\right].
\end{align*}
Therefore, we know
\begin{equation*}
    D_{d+1,d+1} = \bUlast \operatorname{\tr} \left[\E \left(\hLambda_{\tau}\right)^2 \bUtl \Lambda (\bUtl)^\top\right].
\end{equation*}

Additionally, we have
\begin{align*}
    2\left[\E \left(w_\tau^\top x_{\tau,\query} H_\tau\right) \param\right]_{(d+1)^2} 
    &= \left[\begin{pmatrix}[1.5]
    \boldsymbol{0}_{d(d+1) \times d(d+1)} & A \\
    A^\top & \boldsymbol{0}_{(d+1) \times (d+1)}
    \end{pmatrix} \cdot \param\right]_{(d+1)^2} \hspace{10ex} \tag{from \eqref{eqn:gradient_second_term}}\\
    &= \begin{pmatrix}[1.5]
        V_1 + V_1^\top & ... & V_d + V_d^\top & 0_{(d+1) \times (d+1)}
    \end{pmatrix}_{d+1:} \cdot U \tag{definition of $A$ in \eqref{eqn:def_A_V}}\\
    &= \sum_{i,j=1}^d \Lambda_i^\top \Lambda_j \bU_{ji} = \operatorname{tr}\left(\Lambda (\bUtl)^\top \Lambda\right). 
\end{align*}
Then, from \eqref{eqn:gradient_original}, we have the dynamics of $\bUlast$ is
\begin{equation}\label{eqn:dyn_Ulast}
    \frac{\mathrm{d}}{\mathrm{d} t} \bUlast = -\operatorname{tr}\left[\bUlast \Gamma \Lambda \bUtl \Lambda (\bUtl)^\top - \Lambda^2 (\bUtl)^\top\right].
\end{equation}
\end{proof}

\subsection{Proof of Lemma \ref{lemma:globalmin}}
Lemma \ref{lemma:globalmin} gives the form of global minima of an equivalent loss function. First, we prove that gradient flow on $L$ defined in \eqref{eqn:population_loss} from the initial values satisfying Assumption \ref{assume_init} is equivalent to gradient flow on another loss function $\tilde \ell$ defined below. Then, we derive an expression for the global minima of this loss function.

First, from the dynamics of gradient flow, we can actually recover the loss function up to a constant. We have the following lemma.
\begin{lemma}[Loss Function]\label{lemma:loss}
    Consider gradient flow over $L$ in \eqref{eqn:def_original_loss} with respect to $\param$ starting from an initial value satisfying Assumption~\ref{assume_init}.
    This is equivalent to doing gradient flow with respect to $\bUtl$ and $\bUlast$ on the  loss function
    \begin{equation}\label{eqn:loss}
        \tilde \ell\left(\bUtl,\bUlast\right) 
        =  \operatorname{tr}\left[\frac{1}{2} \bUlast^2 \Gamma \Lambda \bUtl \Lambda (\bUtl)^\top- \bUlast \Lambda^2 (\bUtl)^\top\right].
    \end{equation}
\end{lemma}
\begin{proof}
    The proof is simply by taking gradient of the loss function in \eqref{eqn:loss}. For techniques in matrix derivatives, see Lemma \ref{lem:matrix}. We take the gradient of $\tilde \ell$ on $\bUtl$ to obtain
    \begin{equation*}
        \frac{\partial \tilde \ell}{\partial \bUtl}
        = \frac{1}{2} \bUlast^2 \Lambda^\top \Gamma^\top \bUtl \Lambda^\top + \frac{1}{2} \bUlast^2 \Gamma \Lambda \bUtl \Lambda -  \bUlast \Lambda^2 
        = \bUlast^2 \Gamma \Lambda \bUtl \Lambda - \bUlast \Lambda^2,
    \end{equation*}
    since $\Gamma$ and $\Lambda$ are commutable. We take derivatives w.r.t. $\bUlast$ to get
    \begin{equation*}
        \frac{\partial \tilde \ell}{\partial \bUlast} 
        = \operatorname{tr}\left[\bUlast \Gamma \Lambda \bUtl \Lambda (\bUtl)^\top - \Lambda^2 (\bUtl)^\top\right].
    \end{equation*}
    Combining this with Lemma \ref{lem:dyn_finite_data_main}, we have
    \begin{equation*}
        \frac{\mathrm{d}}{\mathrm{d} t} \bUtl(t) = - \frac{\partial \tilde \ell}{\partial \bUtl}, \quad
        \frac{\mathrm{d}}{\mathrm{d} t} \bUlast(t) = - \frac{\partial \tilde \ell}{\partial \bUlast}.
    \end{equation*}
\end{proof}

\ 

We remark that actually this is the loss function $L$ up to some constant. This loss function $\tilde \ell$ can be negative. But we can still compute its global minima as follows.
\begin{corollary}[Minimum of Loss Function]\label{coro:min_loss}
    The loss function $\tilde \ell$ in Lemma \ref{lemma:loss} satisfies
    \begin{equation*}
        \min_{\bUtl \in \mathbb{R}^{d \times d}, \bUlast \in \mathbb{R}} \tilde \ell\left(\bUtl,\bUlast\right) 
        = - \frac{1}{2} \operatorname{tr}\left[\Lambda^2 \Gamma^{-1}\right]
    \end{equation*}
    and 
    \begin{align*}
        &\tilde \ell\left(\bUtl,\bUlast\right) - \min_{\bUtl \in \mathbb{R}^{d \times d}, \bUlast \in \mathbb{R}} \tilde \ell\left(\bUtl,\bUlast\right)
        = \frac{1}{2} \left\|\Gamma^{\frac{1}{2}} \left(\bUlast \Lambda^{\frac{1}{2}} \bUtl \Lambda^{\frac{1}{2}} - \Lambda \Gamma^{-1}\right)\right\|_F^2.
    \end{align*}
\end{corollary}

\begin{proof}
    First, we claim that  
    \begin{equation*}
        \tilde \ell \left(\bUtl,\bUlast\right) 
        = \frac{1}{2} \operatorname{tr}\left[\Gamma \cdot \left(\bUlast \Lambda^{\frac{1}{2}} \bUtl \Lambda^{\frac{1}{2}} - \Lambda \Gamma^{-1}\right) \left(\bUlast \Lambda^{\frac{1}{2}} \bUtl \Lambda^{\frac{1}{2}} - \Lambda \Gamma^{-1}\right)^\top\right] - \frac{1}{2} \operatorname{tr}\left[\Lambda^2 \Gamma^{-1}\right].
    \end{equation*}
    To calculate this, we just need to expand the terms in the brackets and notice that $\Gamma$ and $\Lambda$ are commutable:
    \begin{align*}
        &\quad  \operatorname{tr}\left[\Gamma \cdot \left(\bUlast \Lambda^{\frac{1}{2}} \bUtl \Lambda^{\frac{1}{2}} - \Lambda \Gamma^{-1}\right) \left(\bUlast \Lambda^{\frac{1}{2}} \bUtl \Lambda^{\frac{1}{2}} - \Lambda \Gamma^{-1}\right)^\top\right] - \operatorname{tr}\left[\Lambda^2 \Gamma^{-1}\right]\\
        &\overset{(i)}= \operatorname{tr}\left[\Gamma \cdot \left(\bUlast^2  \Lambda^{\frac{1}{2}} \bUtl \Lambda (\bUtl)^\top \Lambda^{1/2} - \bUlast \Lambda \Gamma^{-1} \Lambda^{\f 12} \bUtl \Lambda^{\f 12} - \bUlast \Lambda^{\f 12} \bUtl \Lambda^{\f 32} \Gamma^{-1} + \Gamma^{-2} \Lambda^{2} \right) \right]  - \tr[\Lambda^2 \Gamma^{-1}]\\
        &= \operatorname{tr}\left[\Gamma \cdot \left(\bUlast^2  \Lambda^{\frac{1}{2}} \bUtl \Lambda (\bUtl)^\top \Lambda^{1/2} - \bUlast \Lambda \Gamma^{-1} \Lambda^{\f 12} \bUtl \Lambda^{\f 12} - \bUlast \Lambda^{\f 12} \bUtl \Lambda^{\f 32} \Gamma^{-1} \right) \right]\\
        &= \bUlast^2\tr \l[  \Gamma \Lambda^{\f 12} \bUtl \Lambda (\bUtl)^\top \Lambda^{\f 12} \r] - \bUlast \tr \l[ \Gamma \Lambda \Gamma^{-1} \Lambda^{\f 12} \bUtl \Lambda^{\f 12} - \Gamma \Lambda^{\f 12} \bUtl \Lambda^{\f 32} \Gamma^{-1} \r]\\
        &\overset{(ii)}= \bUlast^2 \tr \l[ \Gamma \Lambda \bUtl \Lambda (\bUtl)^\top \r] -2 \bUlast \tr \l[ \Lambda^2 \bUtl \Lambda^{\f 12}\r]\\
        &= 2 \tilde \ell\left(\bUtl,\bUlast\right) .
    \end{align*}
Equations $(i)$ and $(ii)$ use that $\Gamma$ and $\Lambda$ commute. 
    
    Since $\Gamma \succeq 0$ and 
    $\left(\bUlast \Lambda^{\frac{1}{2}} \bUtl \Lambda^{\frac{1}{2}} - \Lambda \Gamma^{-1}\right) \left(\bUlast \Lambda^{\frac{1}{2}} \bUtl \Lambda^{\frac{1}{2}} - \Lambda \Gamma^{-1}\right)^\top \succeq 0,$ we know from Lemma \ref{lem:psd} that
    $$
    \frac{1}{2} \operatorname{tr}\left[\Gamma \cdot \left(\bUlast \Lambda^{\frac{1}{2}} \bUtl \Lambda^{\frac{1}{2}} - \Lambda \Gamma^{-1}\right) \left(\bUlast \Lambda^{\frac{1}{2}} \bUtl \Lambda^{\frac{1}{2}}- \Lambda \Gamma^{-1}\right)^\top\right] \geq 0,
    $$
    which implies
    \begin{equation*}
        \tilde \ell \left(\bUtl,\bUlast\right) \geq - \frac{1}{2} \operatorname{tr}\left[\Lambda^2 \Gamma^{-1}\right].
    \end{equation*}
    Equality holds when 
    \begin{equation*}
        \bUtl = \Gamma^{-1}, \quad
        \bUlast = 1,
    \end{equation*}
    so the minimum of $\tilde \ell$ must be $- \frac{1}{2} \operatorname{tr}\left[\Lambda^2 \Gamma^{-1}\right].$ The expression for $\tilde \ell\left(\bUtl,\bUlast\right) - \min \tilde \ell\left(\bUtl,\bUlast\right)$ comes from the fact that $\operatorname{tr}(A^\top A) = \left\|A\right\|_F^2$ for any matrix $A$.
\end{proof}
Lemma~\ref{lemma:globalmin} is an immediate consequence of Corollary\ref{coro:min_loss}, since the loss will keep the same when we replace $(\bUtl,\bUlast)$ by $(c\bUtl,c^{-1}\bUlast)$ for any non-zero constant $c.$

\subsection{Proof of Lemma \ref{lemma:proxypl}}
In this section, we prove that the dynamical system in Lemma \ref{lem:dyn_finite_data_main} satisfies a PL inequality. Then, the PL inequality naturally leads to the global convergence of this dynamical system.  First, we prove a simple lemma, which says the parameters in the LSA model will keep 'balanced' in the whole trajectory. From the proof of this lemma, we can understand why we assume a balanced parameter at the initial time.

\begin{lemma}[Balanced Parameters]\label{lem:balanced}
    Consider gradient flow over $L$ in \eqref{eqn:def_original_loss} with respect to $\param$ starting from an initial value satisfying Assumption~\ref{assume_init}.
    For any $t\geq 0,$ it holds that
    \begin{equation}\label{eqn:balanced_param}
        \bUlast^2 = \tr\left[\bUtl (\bUtl)^\top\right].
    \end{equation}
\end{lemma}

\begin{proof}
    From Lemma \ref{lem:dyn_finite_data_main}, we multiply the first equation in \eqref{eq:dyn_finite_data} by $(\bUtl)^\top$ from the right to get
    \begin{equation*}
        \left(\frac{\mathrm{d}}{\mathrm{d} t} \bUtl(t)\right) (\bUtl(t))^\top
        = - \bUlast^2 \Gamma \Lambda \bUtl \Lambda (\bUtl)^\top+ \bUlast \Lambda^2 (\bUtl)^\top.
    \end{equation*}
    Also we multiply the second equation in Lemma \ref{lem:dyn_finite_data_main} by $\bUlast$ to obtain
    \begin{equation*}
        \left(\frac{\mathrm{d}}{\mathrm{d} t} \bUlast(t)\right) \bUlast(t) 
        = \operatorname{tr}\left[- \bUlast^2 \Gamma \Lambda \bUtl \Lambda (\bUtl)^\top+ \bUlast \Lambda^2 (\bUtl)^\top\right].
    \end{equation*}
    Therefore, we have
    \begin{equation*}
        \operatorname{tr}\left[\left(\frac{\mathrm{d}}{\mathrm{d} t} \bUtl(t)\right) (\bUtl(t))^\top\right] = \left(\frac{\mathrm{d}}{\mathrm{d} t} \bUlast(t)\right) \bUlast(t).
    \end{equation*}
    Taking the transpose of the equation above and adding to itself gives
    \begin{equation*}
        \frac{\mathrm{d}}{\mathrm{d} t} \operatorname{tr}\left[\bUtl(t) (\bUtl(t))^\top\right] = \frac{\mathrm{d}}{\mathrm{d} t} \left(\bUlast(t)^2\right).
    \end{equation*}
    Notice that from Assumption~\ref{assume_init}, we know that at $t=0$,
    \begin{equation*}
        \bUlast(0)^2 = \sigma^2 = \sigma^2 \operatorname{tr}\left[\Theta \Theta^\top \Theta \Theta^\top\right] = \operatorname{tr}\left[\bUtl(0) (\bUtl(0))^\top\right].
    \end{equation*}
    So for any time $t\geq 0,$ the equation holds.
\end{proof}

\

In order to prove the PL inequality, we first prove an important property which says the trajectories of $\bUlast(t)$ stay away from saddle point at origin. First, we prove that $\bUlast(t)$ will stay positive along the whole trajectory.
\begin{lemma}\label{lem:non-zero}
    Consider gradient flow over $L$ in \eqref{eqn:def_original_loss} with respect to $\param$ starting from an initial value satisfying Assumption~\ref{assume_init}.  
    If the initial scale satisfies
    \begin{equation}\label{eqn:eps_condition}
        0 < \sigma < \sqrt{\frac{2}{\sqrt{d} \left\|\Gamma\right\|_{op}}},
    \end{equation}
    then, for any $t \geq 0,$ it holds that
    \begin{equation*}
        \bUlast > 0.
    \end{equation*}
\end{lemma}

\begin{proof}
    From Lemma \ref{lemma:loss}, we are actually doing gradient flow on the loss $\tilde \ell.$ The loss function is non-increasing, because
    \begin{align*}
        \frac{\mathrm{d}\tilde \ell}{\mathrm{d}t}
        &= \left\langle \frac{\mathrm{d} \bUtl}{\mathrm{d}t}, \frac{\partial \tilde \ell}{\partial \bUtl} \right\rangle 
        + \left\langle \frac{\mathrm{d} \bUlast}{\mathrm{d}t}, \frac{\partial \tilde \ell}{\partial \bUlast} \right\rangle 
        = - \left\|\frac{\mathrm{d} \bUtl}{\mathrm{d}t}\right\|_F^2 - \left\|\frac{\mathrm{d} \bUlast}{\mathrm{d}t}\right\|_F^2 \leq 0.
    \end{align*}
    We notice that when $\bUlast = 0,$ the loss function $\tilde \ell = 0.$ Therefore, as long as $\tilde \ell(\bUtl(0),\bUlast(0)) < 0,$ then for any time, $\bUlast$ will be non-zero. Further, since $\bUlast(0) > 0$ and the trajectory of $\bUlast(t)$ must be continuous, we know $\bUlast(t) > 0$ for any $t \geq 0.$ 
    
    Then, it suffices to prove when $0 < \sigma < \sqrt{\frac{2}{\sqrt{d} \left\|\Gamma\right\|_{op}}}$, it holds that  $\tilde \ell(\bUtl(0),\bUlast(0)) < 0.$ From Assumption \ref{assume_init}, we can calculate the loss function at the initial time:
    \begin{align*}
        \tilde \ell(\bUtl(0),\bUlast(0))
        &= \frac{\sigma^4}{2}\operatorname{tr}\left[\Gamma \Lambda \Theta \Theta^\top \Lambda \Theta \Theta^\top\right] - \sigma^2 \operatorname{tr}\left[\Lambda^2 \Theta \Theta^\top\right].
    \end{align*}
    From the property of trace, we know
    \begin{equation*}
        \operatorname{tr}\left[\Lambda^2 \Theta \Theta^\top\right] = \operatorname{tr}\left[\Lambda \Theta \Theta^\top \Lambda^\top\right] = \left\|\Lambda \Theta\right\|_F^2.
    \end{equation*}
    From Von-Neumann's trace inequality (Lemma \ref{lem:trace}) and the fact that $\left\|\Theta \Theta^\top \right\|_F = 1$, we know
    \begin{equation*}
        \operatorname{tr}\left[\Gamma \Lambda \Theta \Theta^\top \Lambda \Theta \Theta^\top\right] 
        \leq \sqrt{d}\left\|\Lambda \Theta \Theta^\top \Lambda \Theta \Theta^\top\right\|_F \cdot \left\| \Gamma\right\|_{op} 
        \leq \sqrt{d}\left\|\Lambda \Theta \right\|_F^2 \left\|\Theta \Theta^\top \right\|_F \left\| \Gamma\right\|_{op}
        = \sqrt{d}\left\|\Lambda \Theta \right\|_F^2 \left\| \Gamma\right\|_{op}.
    \end{equation*}
    Therefore, we have
    \begin{align*}
        \tilde \ell(\bUtl(0),\bUlast(0))
        &\leq \frac{\sqrt{d}\sigma^4}{2}\left\|\Lambda \Theta \right\|_F^2 \left\| \Gamma\right\|_{op} - \sigma^2  \left\|\Lambda \Theta\right\|_F^2 \\
        &= \frac{\sigma^2}{2}  \left\|\Lambda \Theta\right\|_F^2 \left[\sqrt{d}\sigma^2 \left\| \Gamma\right\|_{op} - 2\right].
    \end{align*}
    From Assumption \ref{assume_init}, we know $\left\|\Lambda \Theta\right\|_F \neq 0$. From \eqref{eqn:def_Gamma}, we know $\left\| \Gamma\right\|_{op} > 0.$ Therefore, when
    \begin{equation*}
        0 < \sigma < \sqrt{\frac{2}{\sqrt{d} \left\|\Gamma\right\|_{op}}},
    \end{equation*}
    we have
    \begin{equation*}
        \tilde \ell(\bUtl(0),\bUlast(0)) < 0.
    \end{equation*}
\end{proof}

\ 

From the lemma above, we can actually further prove that the $\bUlast(t)$ can be lower bounded by a positive constant for any $t \geq 0.$ This will be a critical property to prove the PL inequality. We have the following lemma.

\begin{lemma}\label{lem:lower_bound_W}
    Consider gradient flow over $L$ in \eqref{eqn:def_original_loss} with respect to $\param$ starting from an initial value satisfying Assumption~\ref{assume_init} with initial scale $0 < \sigma < \sqrt{\frac{2}{\sqrt{d} \left\|\Gamma\right\|_{op}}}.$ 
    For any $t \geq 0,$ it holds that
    \begin{equation}\label{eqn:lower_bound_Ulast}
        \bUlast \geq \sqrt{\frac{\sigma^2}{2\sqrt{d} \left\|\Lambda\right\|_{op}^2} \left\|\Lambda \Theta\right\|_F^2 \left[2 - \sqrt{d}\sigma^2 \left\| \Gamma\right\|_{op}\right]} > 0.
    \end{equation}
\end{lemma}

\begin{proof}
    We prove by contradiction. Suppose the claim does not hold. From Lemma \ref{lem:balanced}, we know $\bUlast^2 = \operatorname{tr}\left[\bUtl (\bUtl)^\top\right] = \left\|\bUtl\right\|_F^2.$ From Lemma \ref{lem:non-zero}, we know $\bUlast = \left\|\bUtl\right\|_F.$ Recall the definition of loss function:
    \begin{align*}
        \tilde \ell(\bUtl,\bUlast) = \operatorname{tr}\left[\frac{1}{2} \bUlast^2 \Gamma \Lambda \bUtl \Lambda (\bUtl)^\top- \bUlast \Lambda^2 (\bUtl)^\top\right].
    \end{align*}
    Since $\Gamma \succeq 0, \Lambda \succeq 0,$ and they commute, we know from Lemma \ref{lem:psd} that $\Gamma \Lambda \succeq 0.$ Again, since  $\bUtl \Lambda (\bUtl)^\top = \left(\bUtl \Lambda^{\frac{1}{2}}\right)\left(\bUtl \Lambda^{\frac{1}{2}}\right)^\top \succeq 0,$ from Lemma \ref{lem:psd} we have $\operatorname{tr}\left[\frac{1}{2} \bUlast^2 \Gamma \Lambda \bUtl \Lambda (\bUtl)^\top \right] \geq 0.$ So
    \begin{equation*}
        \tilde \ell(\bUtl,\bUlast) \geq - \operatorname{tr}\left[\bUlast \Lambda^2 (\bUtl)^\top\right].
    \end{equation*}
    From Von-Neumann's trace inequality,
    we know for any $t \geq 0,$
    \begin{equation*}
        - \operatorname{tr}\left[\bUlast \Lambda^2 (\bUtl)^\top\right] \geq - \sqrt{d} \bUlast \left\|\Lambda^2\right\|_{op} \left\|\bUtl\right\|_F = - \sqrt{d} \bUlast^2 \left\|\Lambda\right\|_{op}^2.
    \end{equation*}
    Therefore, under our assumption that the claim does not hold, we have
    \begin{equation*}
        \tilde \ell(\bUtl,\bUlast) \geq - \sqrt{d} \bUlast^2 \left\|\Lambda\right\|_{op}^2 
        >  -\frac{\sigma^2}{2} \left\|\Lambda \Theta\right\|_F^2 \left[2 - \sqrt{d}\sigma^2 \left\| \Gamma\right\|_{op}\right] \geq \tilde \ell(\bUtl(0),\bUlast(0)).
    \end{equation*}
    Here, the last inequality comes from the proof of Lemma \ref{lem:non-zero}. This contradicts the non-increasing property of the loss function in gradient flow.
\end{proof}

\ 

Finally, let's prove the PL inequality and further, the global convergence of gradent flow on the loss function $\tilde \ell$. We recall the stated lemma from the main text.

\plinequality*
\begin{proof}
    From the definition and Lemma \ref{lem:lower_bound_W}, we have
    \begin{align}
        \left\|\nabla \ell(\bUtl,\bUlast)\right\|_2^2 
        &\geq\left\|\frac{\partial \ell}{\partial \bUtl}\right\|_F^2 
        = \left\| \bUlast^2 \Gamma \Lambda \bUtl \Lambda - \bUlast \Lambda^2\right\|_F^2 \notag\\
        &= \bUlast^2 \left\| \Gamma \Lambda^{\frac{1}{2}} \left(\bUlast \Lambda^{\frac{1}{2}} \bUtl \Lambda^{\frac{1}{2}} - \Lambda \Gamma^{-1}\right) \Lambda^{\frac{1}{2}} \right\|_F^2 \notag\\
        &\geq \frac{\sigma^2}{2\sqrt{d}\left\|\Lambda\right\|_{op}^2} \left\|\Lambda \Theta\right\|_F^2 \left[2 - \sqrt{d}\sigma^2 \left\| \Gamma\right\|_{op}\right] \left\| \Gamma \Lambda^{\frac{1}{2}} \left(\bUlast \Lambda^{\frac{1}{2}} \bUtl \Lambda^{\frac{1}{2}} - \Lambda \Gamma^{-1}\right) \Lambda^{\frac{1}{2}} \right\|_F^2. \label{eqn:PL_eq1}
    \end{align}
    To see why the second line is true, recall that $\bUlast \in \mathbb{R}$ and $\Gamma$ and $\Lambda$ commute. The last line comes from the lower bound of $\bUlast$ in Lemma \ref{lem:lower_bound_W}. From Corollary \ref{coro:min_loss}, we know
    \begin{align*}
        \ell - \min_{\bUtl \in \mathbb{R}^{d \times d}, \bUlast \in \mathbb{R}} \ell(\bUtl,\bUlast) 
        &= \frac{1}{2}\operatorname{tr}\left[\Gamma \left(\bUlast \Lambda^{\frac{1}{2}} \bUtl \Lambda^{\frac{1}{2}} - \Lambda \Gamma^{-1}\right) \left(\bUlast \Lambda^{\frac{1}{2}} \bUtl \Lambda^{\frac{1}{2}} - \Lambda \Gamma^{-1}\right)^\top\right] \\
        &= \frac{1}{2} \left\|\Gamma^{\frac{1}{2}} \left(\bUlast \Lambda^{\frac{1}{2}} \bUtl \Lambda^{\frac{1}{2}} - \Lambda \Gamma^{-1}\right)\right\|_F^2.
    \end{align*}
    Therefore, we know that 
    \begin{align}
        \ell - \min_{\bUtl \in \mathbb{R}^{d \times d}, \bUlast \in \mathbb{R}} \ell(\bUtl,\bUlast) 
        &\leq \frac{1}{2} \left\|\Gamma \Lambda^{\frac{1}{2}} \left(\bUlast \Lambda^{\frac{1}{2}} \bUtl \Lambda^{\frac{1}{2}} - \Lambda \Gamma^{-1}\right) \Lambda^{\frac{1}{2}}\right\|_F^2  \cdot \left\|\Gamma^{-\frac{1}{2}} \Lambda^{-\frac{1}{2}}\right\|_F^2 \left\| \Lambda^{-\frac{1}{2}}\right\|_F^2 \notag \\
        &= \frac{1}{2} \left\|\Gamma \Lambda^{\frac{1}{2}} \left(\bUlast \Lambda^{\frac{1}{2}} \bUtl \Lambda^{\frac{1}{2}} - \Lambda \Gamma^{-1}\right) \Lambda^{\frac{1}{2}}\right\|_F^2  \cdot \operatorname{tr}\left(\Gamma^{-1} \Lambda^{-1}\right)\operatorname{tr}\left(\Lambda^{-1}\right) \label{eqn:PL_eq2}
    \end{align}
    We compare \eqref{eqn:PL_eq1} and \eqref{eqn:PL_eq2} to obtain that in order to make the PL condition hold, one needs to let
    \begin{equation*}
        \mu := \frac{\sigma^2}{\sqrt{d}\left\|\Lambda\right\|_{op}^2 \operatorname{tr}\left(\Gamma^{-1} \Lambda^{-1}\right)\operatorname{tr}\left(\Lambda^{-1}\right)} \left\|\Lambda \Theta\right\|_F^2 \left[2 - \sqrt{d}\sigma^2 \left\| \Gamma\right\|_{op}\right] > 0.
    \end{equation*}
    Once we set this $\mu,$ we get the PL inequality. The $\mu$ is positive due to the assumption for $\sigma$ in the lemma.
    
    From the dynamics of gradient flow and the PL condition, we know
    \begin{align*}
        \frac{\mathrm{d}}{\mathrm{d}t} \left(\tilde \ell - \min_{\bUtl \in \mathbb{R}^{d \times d}, \bUlast \in \mathbb{R}} \tilde \ell(\bUtl,\bUlast)\right)
        &= \left\langle \frac{\mathrm{d} \bUtl}{\mathrm{d}t}, \frac{\partial \tilde \ell}{\partial \bUtl} \right\rangle 
        + \left\langle \frac{\mathrm{d} \bUlast}{\mathrm{d}t}, \frac{\partial \tilde \ell}{\partial \bUlast} \right\rangle 
        = - \left\|\frac{\mathrm{d} \bUtl}{\mathrm{d}t}\right\|_F^2 - \left|\frac{\mathrm{d} \bUlast}{\mathrm{d}t}\right|^2 \\
        &\leq -\mu \left(\tilde \ell - \min_{\bUtl \in \mathbb{R}^{d \times d}, \bUlast \in \mathbb{R}} \tilde \ell(\bUtl,\bUlast)\right).
    \end{align*}
    Therefore, we have when $t \to \infty,$
    \begin{equation*}
        0 \leq \tilde \ell - \min_{\bUtl \in \mathbb{R}^{d \times d}, \bUlast \in \mathbb{R}} \tilde \ell(\bUtl,\bUlast) 
        \leq \exp\left(-\mu t\right) \left[\tilde \ell(\bUtl(0),\bUlast(0)) - \min_{\bUtl \in \mathbb{R}^{d \times d}, \bUlast \in \mathbb{R}} \tilde \ell(\bUtl,\bUlast)\right] \to 0,
    \end{equation*}
    which implies
    \begin{equation*}
        \lim_{t \to \infty} \left[\tilde \ell - \min_{\bUtl \in \mathbb{R}^{d \times d}, \bUlast \in \mathbb{R}} \tilde \ell(\bUtl,\bUlast)\right]  = 0.
    \end{equation*}
From Corollary \ref{coro:min_loss}, we know this is
    \begin{equation*}
        \left\|\Gamma^{\frac{1}{2}} \left(\bUlast \Lambda^{\frac{1}{2}} \bUtl \Lambda^{\frac{1}{2}} - \Lambda \Gamma^{-1}\right)\right\|_F^2 \to 0.
    \end{equation*}
    Since $\Gamma$ and $\Lambda$ are non-singular and positive definite, and they commute, we know
    \begin{equation*}
        \left\|\bUlast  \bUtl - \Gamma^{-1}\right\|_F^2 
        \leq  \left\|\Gamma^{-\frac{1}{2}} \Lambda^{-\frac{1}{2}} \right\|_F^2 \left\|\Gamma^{\frac{1}{2}} \left(\bUlast \Lambda^{\frac{1}{2}} \bUtl \Lambda^{\frac{1}{2}} - \Lambda \Gamma^{-1}\right)\right\|_F^2 \left\|\Lambda^{-\frac{1}{2}} \right\|_F^2
        \to 0.
    \end{equation*}
    This implies
    $\bUlast  \bUtl - \Gamma^{-1} \to 0_{d \times d}$ entry-wise. Since $\bUlast = \left\|\bUtl\right\|_F,$ we know 
    \begin{equation*}
         \bUlast^2 = \left\|\bUlast  \bUtl\right\|_F \to \left\|\Gamma^{-1}\right\|_F.
    \end{equation*}
    Therefore, we know
    \begin{align*}
    \lim_{t \to \infty} \bUlast(t) = \left\|\Gamma^{-1}\right\|_F^{\frac{1}{2}} \text{ and }
        \lim_{t \to \infty} \bUtl(t) = \left\|\Gamma^{-1}\right\|_F^{-\frac{1}{2}} \Gamma^{-1}.
    \end{align*}
\end{proof}

\section{Proof of Theorem \ref{thm:general_risk_maintext}}\label{appendix_proof_risk_general}

In this section, we prove Theorem \ref{thm:general_risk_maintext}, which characterizes the excess risk of the prediction of a trained LSA layer with respect to the risk of best linear predictor, on a new task which is possibly non-linear. First, we restate the theorem.

\generalrisk*
\begin{proof}
    Unless otherwise specified, we denote $\E$ as the expectation over $(\testx_i,\testy_i), (\testx_\query,\testy_\query) \iid \calD.$ Since when $(x,y) \sim \calD,$ we assume $\E[x], \E[y], \E[xy], \E[xx^\top], \E[y^2 xx^\top]$ exist, we know that $\E \left(\sip{\testw}{\testx_\query} - \testy_\query\right)^2$ exists for each $\testw \in \mathbb{R}^d.$ We denote
    \begin{equation*}
        a := \mathop{\arg \min}_{\testw \in \mathbb{R}^d} \E \left(\sip{\testw}{\testx_\query} - \testy_\query\right)^2
    \end{equation*}
    as the weight of the best linear approximator. Actually, if we denote the function inside the minimum above as $R(\testw),$ we can write it as
    \begin{equation*}
        R(\testw) = \testw^\top \Lambda \testw - 2 \E \left(\testy_\query \cdot \testx_\query^\top \right) \testw + \E \testy_\query^2.
    \end{equation*}
    Since the Hessian matrix $\frac{\partial^2}{\partial \testw \partial \testw^\top} R(\testw)$ is $\Lambda$, which is positive definitive, we know that this function is strictly convex and hence, the global minimum can be achieved at the unique first-order stationary point. This is
    \begin{equation}
        a = \Lambda^{-1} \E \left(\testy_\query \cdot \testx_\query \right).
    \end{equation}
    We also define a similar vector for ease of computation:
    \begin{equation}
        b = \Gamma^{-1} \E \left(\testy_\query \cdot \testx_\query \right).
    \end{equation}
    Therefore, we can decompose the error as
    \begin{align*}
        &\E\left(\widehat{y}_{\query} - \testy_\query\right)^2
         = \underbrace{\E \left(\sip{a}{\testx_\query} - \testy_\query\right)^2}_{\text{I}}
        + \underbrace{\E \left(\widehat{y}_{\query} - \sip{b}{\testx_\query}\right)^2}_{\text{II}} \\
        &\hspace{5ex} + \underbrace{\E \left(\sip{b}{\testx_\query} - \sip{a}{\testx_\query}\right)^2}_{\text{III}}
        + \underbrace{2\E \left(\widehat{y}_{\query} - \sip{b}{\testx_\query}\right)\left(\sip{a}{\testx_\query} - \testy_\query\right)}_{\text{IV}} \\
        &\hspace{5ex} + \underbrace{2\E \left(\widehat{y}_{\query} - \sip{b}{\testx_\query}\right)\left(\sip{b}{\testx_\query} - \sip{a}{\testx_\query} \right)}_{\text{V}}
        + \underbrace{2\E \left(\sip{b}{\testx_\query} - \sip{a}{\testx_\query} \right)\left(\sip{a}{\testx_\query} - \testy_\query\right)}_{\text{VI}}
    \end{align*}
    The term I is the first term on the right hand side of \eqref{eqn:risk_nonlinear}. So it suffices to calculate II to VI. 
    
    \ 
    
    First, from the tower property of conditional expectation, we have
    \begin{align*}
        \text{V}
        &= 2\E \left[ \E \left(\left(\widehat{y}_{\query} - \sip{b}{\testx_\query}\right)\left(\sip{b}{\testx_\query} - \sip{a}{\testx_\query} \right) \bigg| \testx_\query\right) \right] \\
        &= 2\E \left[ \E \left(\widehat{y}_{\query} - \sip{b}{\testx_\query} \bigg| \testx_\query \right) \left(\sip{b}{\testx_\query} - \sip{a}{\testx_\query} \right) \right] = 0,
    \end{align*}
    since
    \begin{equation*}
        \E \left(\widehat{y}_{\query} - \sip{b}{\testx_\query} \bigg| \testx_\query \right)
        = \left(\E \frac{1}{M}\sum_{i=1}^M \testy_i \Gamma^{-1} \testx_i  - b\right)^\top \testx_\query = 0.
    \end{equation*}
    
    \ 
    
    Similarly, for IV, we have
    \begin{align*}
        \text{IV}
        &= 2\E \left(\widehat{y}_{\query} - \sip{b}{\testx_\query}\right)\left(\sip{a}{\testx_\query} - \testy_\query\right) \\
        &= 2\E \left[ \E \left(\left(\widehat{y}_{\query} - \sip{b}{\testx_\query}\right)\left(\sip{a}{\testx_\query} - \testy_\query\right) \bigg| \testx_\query, \testy_\query\right)\right] \\
        &= 2\E \left[ \E \left(\widehat{y}_{\query} - \sip{b}{\testx_\query}\bigg| \testx_\query, \testy_\query\right) \left(\sip{a}{\testx_\query} - \testy_\query\right) \right] \\
        &= 0.
    \end{align*}
    
    \ 
    
    For VI, we have
    \begin{align*}
        \text{VI}
        &=2 \E \tr \left[(b-a) \left(\sip{a}{\testx_\query} - \testy_\query\right) \testx_\query^\top \right] \\
        &= 2 \tr \left[(b-a) a^\top \Lambda\right] - 2 \tr\left[(b-a) \E \left(\testy_\query \testx_\query^\top \right) \right]
        = 0,
    \end{align*}
    where the last line comes from the definition of $a.$ Therefore, all cross terms vanish and it suffices to consider II and III.
    
    \ 
    
    For II, from the definition we have
    \begin{align*}
        &\text{II} \\
        =&\E \left(\frac{1}{M}\sum_{i=1}^M \testy_i \testx_i  - \E \left(\testy_\query \cdot \testx_\query \right) \right)^\top \Gamma^{-1} \testx_\query \testx_\query^\top \Gamma^{-1} \left(\frac{1}{M}\sum_{i=1}^M \testy_i \testx_i  - \E \left(\testy_\query \cdot \testx_\query \right) \right) \\
        =& \E \tr \left(\frac{1}{M}\sum_{i=1}^M \testy_i \testx_i  - \E \left(\testy_\query \cdot \testx_\query \right) \right) \left(\frac{1}{M}\sum_{i=1}^M \testy_i \testx_i  - \E \left(\testy_\query \cdot \testx_\query \right) \right)^\top \Gamma^{-2} \Lambda \tag{property of trace and the fact that $\Gamma$ and $\Lambda$ commute} \\
        =& \frac{1}{M^2} \sum_{i,j=1}^M \E \tr \left\{\left(\testy_i \testx_i  - \E \left(\testy_\query \cdot \testx_\query \right) \right) \left( \testy_j \testx_j  - \E \left(\testy_\query \cdot \testx_\query \right) \right)^\top \Gamma^{-2} \Lambda\right\} \\
        =& \frac{1}{M} \E \tr \left\{\left(\testy_1 \testx_1  - \E \left(\testy_\query \cdot \testx_\query \right) \right) \left( \testy_1 \testx_1  - \E \left(\testy_\query \cdot \testx_\query \right) \right)^\top \Gamma^{-2} \Lambda\right\} \tag{all cross terms vanish due to the independence of $\testx_i$} \\
        =& \frac{1}{M} \tr\left[\Sigma \Gamma^{-2} \Lambda\right].
    \end{align*}
    The last line comes from the definition of $\Sigma.$ 
    
    \ 
    
    For III, we have
    \begin{align*}
        \text{III}
        &= \E (b-a)^\top \testx_\query \testx_\query^\top (b-a)
        = a^\top \Lambda (\Gamma^{-1}-\Lambda^{-1}) \Lambda (\Gamma^{-1}-\Lambda^{-1}) \Lambda a \\
        &= \tr\left[\left(I - \Gamma \Lambda^{-1}\right)^2 \Gamma^{-2} \Lambda^3 a a^\top\right] \tag{property of trace and the fact that $\Gamma$ and $\Lambda$ commute}\\
        &= \frac{1}{N^2} \tr\left[\left(I_d + \tr(\Lambda) \Lambda^{-1}\right)^2 \Gamma^{-2} \Lambda^3 a a^\top\right] \\
        &= \frac{1}{N^2} \left[\tr(\Gamma^{-2} \Lambda^3 a a^\top) + 2\tr(\Lambda) \tr(\Gamma^{-2} \Lambda^2 a a^\top) 
        + \tr(\Lambda)^2 \tr(\Gamma^{-2} \Lambda a a^\top)\right].
    \end{align*}
    Combining all terms above, we conclude.
\end{proof}

\section{Proof of Theorem \ref{thm:convergence_random}}\label{appendix_proof_random}
The proof of Theorem \ref{thm:convergence_random} is very similar to that of Theorem \ref{thm:mainresult}. The first step is to explicitly write out the dynamical system. In order to do so, we notice that the Lemma \ref{lemma:lsa.is.quadratic} does not depend on the training data and data-generaing distribution and hence, it still holds in the case of a random covariance matrix. Therefore, we know when we input the embedding matrix $E_\tau$ to the linear self-attention layer with parameter $\params = (\WKQ,\WPV),$ the prediction will be 
\begin{equation*}
    \widehat y_{\query}(E_\tau; \params) = \param^\top H_\tau \param,
\end{equation*}
where the matrix $H_\tau$ is defined as,
\begin{equation*}
    H_\tau = \frac{1}{2} X_\tau \otimes \left(\frac{E_\tau E_\tau^\top}{N}\right) \in \mathbb{R}^{(d+1)^2 \times (d+1)^2},\quad
    X_\tau = \begin{pmatrix}[1.5]
    0_{d \times d} & x_{\tau,\query}\\
    \left(x_{\tau,\query}\right)^\top & 0
    \end{pmatrix} \in \mathbb{R}^{(d+1) \times (d+1)}
\end{equation*}
and
\begin{equation*}
    \param = \vector(\bU) \in \mathbb{R}^{(d+1)^2}, \quad
    \bU = \begin{pmatrix}[1.5]
            \bUtl & \bUtr \\
            (\bUbl)^\top & \bUlast
    \end{pmatrix} \in \mathbb{R}^{(d+1) \times (d+1)},
\end{equation*}
where $\bUtl = \WKQ_{11} \in \mathbb{R}^{d \times d}, \bUtr = w_{21}^{PV} \in \mathbb{R}^{d \times 1},\bUbl = w_{21}^{KQ} \in \mathbb{R}^{d \times 1},\bUlast = w_{22}^{PV} \in \mathbb{R}$ correspond to particular components of 
$\WPV$ and $\WKQ$, defined in \eqref{eqn:block_matrix}.

\ 

\subsection{Dynamical system}
The next lemma gives the dynamical system when the covariance matrices in the prompts are i.i.d. sampled from some distribution. Notice that in the lemma below, we do not assume $\Lambda_\tau$ are almost surely diagonal. The case when the covariance matrices are diagonal can be viewed as a special case of the following lemma.

\begin{restatable}{lemma}{dyn_system_random}\label{lem:dyn_random}
    Consider gradient flow on \eqref{eqn:loss_random_cov} with respect to $\param$ starting from an initial value that satisfies Assumption~\ref{assume_init}.
    We assume the covariance matrices $\Lambda_\tau$ are sampled from some distribution with finite third moment and $\Lambda_\tau$ are positive definite almost surely. We denote 
    $\param = \vecU := \vector \begin{pmatrix}[1.5]
                \bUtl & \bUtr \\
                (\bUbl)^\top & \bUlast
        \end{pmatrix}$ 
        and define
    \begin{equation*}
        \Gamma_\tau = \left(1 + \frac{1}{N}\right)\Lambda_\tau + \frac{1}{N}\operatorname{tr}(\Lambda_\tau) I_d \in \mathbb{R}^{d \times d}.
    \end{equation*}
    Then the dynamics of $\bU$ follows
    \begin{equation}\label{eqn:random_dyn}
    \begin{aligned}
        \frac{\mathrm{d}}{\mathrm{d} t} \bUtl(t)
        &= - \bUlast^2 \E \left[\Gamma_\tau \Lambda_\tau \bUtl \Lambda_\tau \right] + \bUlast \E \left[\Lambda_\tau^2\right]\\
        \frac{\mathrm{d}}{\mathrm{d} t} \bUlast(t) &= -\bUlast \operatorname{\tr} \E \left[\Gamma_\tau \Lambda_\tau \bUtl \Lambda_{\tau} (\bUtl)^\top\right] + \operatorname{tr}\left(\E \left[\Lambda_\tau^2\right] (\bUtl)^\top \right),
    \end{aligned}
    \end{equation}
    and $\bUtr (t)= 0_{d}, \bUbl (t)= 0_{d}$ for all $t \geq 0.$
\end{restatable}

\begin{proof}
    This lemma is a natural corollary of Lemma \ref{lem:dyn_finite_data_main}. Notice that Lemma~\ref{lem:dyn_finite_data_main} holds for any fixed positive definite $\Lambda_\tau.$ So when $\Lambda_\tau$ is random, if we condition on $\Lambda_\tau$, the dynamical system will be 
    \begin{equation}
    \begin{aligned}
        \frac{\mathrm{d}}{\mathrm{d} t} \bUtl(t)
        &= - \bUlast^2  \left[\Gamma_\tau \Lambda_\tau \bUtl \Lambda_\tau \right] + \bUlast  \left[\Lambda_\tau^2\right]\\
        \frac{\mathrm{d}}{\mathrm{d} t} \bUlast(t) &= -\bUlast \operatorname{\tr} \left[\Gamma_\tau \Lambda_\tau \bUtl \Lambda_{\tau} (\bUtl)^\top\right] + \operatorname{tr}\left(\left[\Lambda_\tau^2\right] (\bUtl)^\top \right),
    \end{aligned}
    \end{equation}
    and $\bUtr (t)= 0_{d}, \bUbl (t)= 0_{d}$ for all $t \geq 0.$ Then, we conclude by simply taking expectation over $\Lambda_\tau.$
\end{proof}

\ 

The lemma above gives the dynamical system with general random covariance matrix. When $\Lambda_\tau$ are diagonal almost surely, we can actually simplify the dynamical system above. In this case, we have the following corollary.
\begin{corollary}\label{coro:dyn_diag_random}
    Under the assumptions of Lemma \ref{lem:dyn_random}, we further assume the covariance matrix $\Lambda_\tau$ to be diagonal almost surely. We denote $u_{ij}(t) \in \R$ as the $(i,j)$-th entry of $\bUtl(t),$ and further denote
    \begin{equation}\label{eqn:def_gamma_zeta_xi}
    \begin{aligned}
        \gamma_i &= \E \left[\frac{N+1}{N}\lambda_{\tau,i}^3 + \frac{1}{N}\lambda_{\tau,i}^2 \cdot \sum_{j=1}^d \lambda_{\tau,j}\right], \\
        \xi_i &= \E \left[\lambda_{\tau,i}^2\right], \\
        \zeta_{ij} &= \E \left[\frac{N+1}{N}\lambda_{\tau,i}^2 \lambda_{\tau,j} + \frac{1}{N}\lambda_{\tau,i} \lambda_{\tau,j} \cdot \sum_{k=1}^d \lambda_{\tau,k}\right]
    \end{aligned}
    \end{equation}
    for $i,j \in [d],$ where the expectation is over the distribution of $\Lambda_\tau.$ Then, the dynamical system \eqref{eqn:random_dyn} is equivalent to
    \begin{equation}\label{eqn:dyn_random_diag}
    \begin{aligned}
        \frac{\mathrm{d}}{\mathrm{d}t} u_{ii}(t) &= - \gamma_i \bUlast^2 u_{ii} + \xi_i \bUlast \quad \forall i \in [d], \\
        \frac{\mathrm{d}}{\mathrm{d}t} u_{ij}(t) &= - \zeta_{ij} \bUlast^2 u_{ij} \quad \forall i \neq j \in [d], \\
        \frac{\mathrm{d}}{\mathrm{d} t} \bUlast(t) &= - \sum_{i=1}^d \left[\gamma_i \bUlast u_{ii}^2\right] - \sum_{i\neq j} \zeta_{ij} \bUlast u_{ij}^2 + \sum_{i=1}^d \left[\xi_i u_{ii}\right].
    \end{aligned}
    \end{equation}
\end{corollary}

\begin{proof}
    This is directly obtained by rewriting the equation for each entry of $\bUtl$ and recalling the assumption that $\Lambda_\tau$ (and hence $\Gamma_\tau$) is diagonal almost surely.
\end{proof}

\subsection{Loss function and global minima}
As in the proof of Theorem \ref{thm:mainresult}, we can actually recover the loss function in the random covariance case, up to a constant.
\def\rdm{\mathsf{rdm}}

\begin{lemma}\label{lem:loss_random_cov}
    The differential equations in \eqref{eqn:dyn_random_diag} are equivalent to gradient flow on the loss function
    \begin{equation}\label{eqn:def_loss_random}
    \begin{aligned}
        \ell_{\rdm}(\bUtl, \bUlast)
        &= \E \tr \left[\frac{1}{2} \bUlast^2 \Gamma_\tau \Lambda_\tau \bUtl \Lambda_\tau (\bUtl)^\top - \bUlast \Lambda_\tau^2 (\bUtl)^\top\right] \\
        &= \frac{1}{2} \sum_{i=1}^d \left[\gamma_i \bUlast^2 u_{ii}^2\right] + \frac{1}{2} \sum_{i\neq j} \zeta_{ij} \bUlast^2 u_{ij}^2 - \sum_{i=1}^d \left[\xi_i u_{ii} \bUlast\right]
    \end{aligned}
    \end{equation}
    with respect to $u_{ij} \forall i,j \in [d]$ and $u_{-1}$, from an initial value that satisfies Assumption~\ref{assume_init}.
\end{lemma}

\begin{proof}
This can be verified by simply taking gradient of $\ell_{\rdm}$ to show that
\begin{equation*}
    \frac{\mathrm{d}}{\mathrm{d}t} u_{ii} = - \frac{\partial \ell_{\rdm}}{\partial u_{ii}} \quad \forall i \in [d], \quad
    \frac{\mathrm{d}}{\mathrm{d}t} u_{ij} = - \frac{\partial \ell_{\rdm}}{\partial u_{ij}} \quad \forall i\neq j \in [d], \quad
    \frac{\mathrm{d}}{\mathrm{d}t} \bUlast = - \frac{\partial \ell_{\rdm}}{\partial \bUlast}.
\end{equation*}
\end{proof}

\ 

Next, we solve for the minimum of $\ell_{\rdm}$ and give the expression for all global minima.
\begin{lemma}\label{lem:min_loss_random}
    Let $\ell_{\rdm}$ be the loss function in \eqref{eqn:def_loss_random}.
    We denote
    \begin{equation*}
        \min \ell_{\rdm} := \min_{\bUtl \in \R^{d\times d}, \bUlast \in \R} \ell_{\rdm}\left(\bUtl, \bUlast\right).
    \end{equation*}
    Then, we have
    \begin{equation}
        \min \ell_{\rdm} = - \frac{1}{2} \sum_{i=1}^d \frac{\xi_i^2}{\gamma_i}
    \end{equation}
    and 
    \begin{equation}
        \ell_{\rdm}(\bUtl,\bUlast) - \min \ell_{\rdm} = \frac{1}{2} \sum_{i=1}^d \gamma_i \left(u_{ii} \bUlast - \frac{\xi_i}{\gamma_i}\right)^2 + \frac{1}{2} \sum_{i\neq j} \zeta_{ij} \bUlast^2 u_{ij}^2.
    \end{equation}
    Moreover, denoting $u_{ij}$ as the $(i,j)$-entry of $\bUtl$, all global minima of $\ell_\rdm$ satisfy
    \begin{equation}\label{eqn:global_min_random}
        \bUlast \cdot u_{ij} = \mathbb{I}(i=j) \cdot \frac{\xi_i}{\gamma_i}.
    \end{equation}
\end{lemma}

\begin{proof}
    From the definition of $\ell_\rdm,$ we have
    \begin{align*}
        \ell_\rdm = \frac{1}{2} \sum_{i=1}^d \gamma_i \left(u_{ii} \bUlast - \frac{\xi_i}{\gamma_i}\right)^2 + \frac{1}{2} \sum_{i\neq j} \zeta_{ij} \bUlast^2 u_{ij}^2 - \frac{1}{2} \sum_{i=1}^d \frac{\xi_i^2}{\gamma_i} 
        \geq -\frac{1}{2} \sum_{i=1}^d \frac{\xi_i^2}{\gamma_i}.
    \end{align*}
    The equation holds when $u_{ij} = 0$ for $i \neq j \in [d]$ and $\bUlast u_{ii} = \frac{\xi_i}{\gamma_i}$ for each $i \in [d].$ This can be achieved by simply letting $\bUlast = 1$ and $u_{ii} = \frac{\xi_i}{\gamma_i}$ for $i \in [d].$ Of course, when we replace $(\bUlast, u_{ii})$ with $(c\bUlast, c^{-1} u_{ii})$ for any constant $c \neq 0,$ we can also achieve this global minimum.
\end{proof}

\subsection{PL Inequality and global convergence}

Finally, to end the proof, we prove a Polyak-\L ojasiewicz Inequality on the loss function $\ell_\rdm$, and then prove global convergence. Before that, let's first prove the balanced condition of parameters will hold during the whole trajectory. 

\begin{lemma}[Balanced condition]\label{lem:balanced_random}
    Under the assumptions of Lemma \ref{lem:dyn_random}, for any $t \geq 0,$ it holds that 
    \begin{equation}\label{eqn:balanced_param_random}
        \bUlast^2 = \tr\left[\bUtl (\bUtl)^\top\right].
    \end{equation}
\end{lemma}
\begin{proof}
    The proof is similar to the proof of Lemma \ref{lem:balanced}. From Lemma \ref{lem:dyn_finite_data_main}, we multiply the first equation in \eqref{eqn:random_dyn} by $(\bUtl)^\top$ from the right to get
    \begin{equation*}
        \left[\frac{\mathrm{d}}{\mathrm{d} t} \bUtl(t)\right] (\bUtl)^\top
        = - \bUlast^2 \E \left[\Gamma_\tau \Lambda_\tau \bUtl \Lambda_\tau (\bUtl)^\top\right] + \bUlast \E \left[\Lambda_\tau^2 (\bUtl)^\top\right].
    \end{equation*}
    Also we multiply the second equation in Lemma \ref{eqn:random_dyn} by $\bUlast$ to obtain
    \begin{equation*}
        \left(\frac{\mathrm{d}}{\mathrm{d} t}\bUlast(t) \right)\bUlast(t) 
        = -\bUlast^2 \operatorname{\tr} \E \left[\Gamma_\tau \Lambda_\tau \bUtl \Lambda_{\tau} (\bUtl)^\top\right] + \bUlast\operatorname{tr}\left(\E \left[\Lambda_\tau^2\right] (\bUtl)^\top \right),
    \end{equation*}
    Therefore, we have
    \begin{equation*}
        \operatorname{tr}\left[\left(\frac{\mathrm{d}}{\mathrm{d} t} \bUtl(t)\right) (\bUtl(t))^\top\right] = \left(\frac{\mathrm{d}}{\mathrm{d} t} \bUlast(t)\right) \bUlast(t).
    \end{equation*}
    Taking the transpose of the equation above and adding to itself gives
    \begin{equation*}
        \frac{\mathrm{d}}{\mathrm{d} t} \operatorname{tr}\left[\bUtl(t) (\bUtl(t))^\top\right] = \frac{\mathrm{d}}{\mathrm{d} t} \left(\bUlast(t)^2\right).
    \end{equation*}
    Notice that from Assumption~\ref{assume_init}, we know that
    \begin{equation*}
        \bUlast(0)^2 = \sigma^2 = \sigma^2 \operatorname{tr}\left[\Theta \Theta^\top \Theta \Theta^\top\right] = \operatorname{tr}\left[\bUtl(0) (\bUtl(0))^\top\right].
    \end{equation*}
    So for any time $t\geq 0,$ the equation holds.
\end{proof}

\ 

Next, similar to the proof of Theorem \ref{thm:mainresult}, we prove that, as long as the initial scale is small enough, $\bUlast$ will be positive along the whole trajectory and can be lower bounded by a positive constant, which implies that the trajectories will be away from the saddle point at the origin.

\begin{lemma}\label{lem:non_zero_random}
    We do gradient flow on $\ell_\rdm$ with respect to $u_{i,j} \ (\forall i,j \in [d])$ and $\bUlast$. Suppose the initialization satisfies Assumption \ref{assume_init} with initial scale
    \begin{equation}\label{eqn:init_cond_random_proof}
        0 < \sigma < \sqrt{\frac{2\left\|\E \Lambda_\tau \Theta\right\|_F^2}{\sqrt{d} \left[\E \left\|\Gamma_\tau\right\|_{op} \left\|\Lambda_\tau\right\|_F^2 \right]}},
    \end{equation}
    then for any $t \geq 0,$ it holds that
    \begin{equation}
        \bUlast(t) > 0.
    \end{equation}
\end{lemma}

\begin{proof}
    From the dynamics of gradient flow, we know the loss function $\ell_\rdm$ is non-increasing:
    \begin{equation*}
        \frac{\mathrm{d} \ell_\rdm}{\mathrm{d} t} 
        = \sum_{i,j=1}^d \frac{\partial \ell_\rdm}{\partial u_{ij}} \cdot \frac{\mathrm{d} u_{ij}}{\mathrm{d} t} + \frac{\partial \ell_\rdm}{\partial \bUlast} \cdot \frac{\mathrm{d} \bUlast}{\mathrm{d} t}
        = - \sum_{i,j=1}^d \left[\frac{\partial \ell_\rdm}{\partial u_{ij}}\right]^2 - \left[\frac{\partial \ell_\rdm}{\partial \bUlast}\right]^2 \leq 0.
    \end{equation*}
Since we assume $\bUtl(0) = \Theta \Theta^\top,$ we know the loss function at $t = 0$ is
\begin{equation*}
    \ell_\rdm(\bUtl(0),\bUlast(0)) = \E \tr \left[\frac{\sigma^4}{2} \Gamma_\tau \Lambda_\tau \Theta \Theta^\top \Lambda_\tau \Theta \Theta^\top - \sigma^2 \Lambda_\tau^2 \Theta \Theta^\top\right].
\end{equation*}
From the property of trace, we know
\begin{equation*}
    \E \tr\left[\sigma^2 \Lambda_\tau^2 \Theta \Theta^\top\right] = \sigma^2 \left\|\E \Lambda_\tau \Theta\right\|_F^2.
\end{equation*}
From Von-Neumann's trace inequality and the assumption that $\left\| \Theta \Theta^\top\right\|_F = 1$, we know
\begin{align*}
    \E \tr \left[\frac{\sigma^4}{2} \Gamma_\tau \Lambda_\tau \Theta \Theta^\top \Lambda_\tau \Theta \Theta^\top \right] 
    &\leq \frac{\sigma^4 \sqrt{d}}{2} \E \left\|\Gamma_\tau\right\|_{op} \left\|\Lambda_\tau \Theta \Theta^\top \Lambda_\tau \Theta \Theta^\top\right\|_F \\
    &\leq \frac{\sigma^4 \sqrt{d} \left\| \Theta \Theta^\top\right\|_F^2}{2} \left[\E \left\|\Gamma_\tau\right\|_{op} \left\|\Lambda_\tau\right\|_F^2 \right]
    = \frac{\sigma^4 \sqrt{d}}{2} \left[\E \left\|\Gamma_\tau\right\|_{op} \left\|\Lambda_\tau\right\|_F^2 \right].
\end{align*}
From the assumptions on $\Theta$ and $\Lambda_\tau$ we know $\E \Lambda_\tau \Theta \neq 0_{d \times d}$ and $\E \left\|\Gamma_\tau\right\|_{op} \left\|\Lambda_\tau\right\|_F^2 > 0.$ Therefore, comparing the two displays above, we know when \eqref{eqn:init_cond_random_proof} holds, we must have $\ell_\rdm(0) < 0.$ So from the non-increasing property of the loss function, we know $\ell_\rdm(t) < 0$ for any time $t \geq 0.$ Notice that when $\bUlast = 0,$ the loss function is also zero, which suggests that $\bUlast(t) \neq 0$ for any time $t \geq 0.$ Since $\bUlast(0) > 0$ and the trajectory of $\bUlast$ must be continuous, we know that it stays positive at all times. 
\end{proof}

\ 

\begin{lemma}\label{lem:lower_bound_u_random}
    We do gradient flow on $\ell_\rdm$ with respect to $u_{i,j} \ (\forall i,j \in [d])$ and $\bUlast$. Suppose the initialization satisfies Assumption \ref{assume_init} and the initial scale satisfies \eqref{eqn:init_cond_random_proof}. Then, for any $t \geq 0,$ it holds that
    \begin{equation}
        \bUlast(t) \geq \sqrt{\frac{\sigma^2}{2\sqrt{d}\left\|\E \Lambda_\tau^2\right\|_{op}} \left[2\left\|\E \Lambda_\tau \Theta\right\|_F^2 - \sqrt{d} \sigma^2 \left[\E \left\|\Gamma_\tau\right\|_{op} \left\|\Lambda_\tau\right\|_F^2 \right]\right]} > 0.
    \end{equation}
\end{lemma}

\begin{proof}
    From the dynamics of gradient flow, we know $\ell_\rdm$ is non-increasing (see the proof of Lemma \ref{lem:non_zero_random}). Recall the definition of the loss function:
    \begin{equation*}
        \ell_{\rdm} (\bUtl, \bUlast)
        = \E \tr \left[\frac{1}{2} \bUlast^2 \Gamma_\tau \Lambda_\tau \bUtl \Lambda_\tau (\bUtl)^\top - \bUlast \Lambda_\tau^2 (\bUtl)^\top\right].
    \end{equation*}
    Since $\Lambda_\tau$ commutes with $\Gamma_\tau$ and they are both positive definite almost surely, we know that $\Gamma_\tau \Lambda_\tau \succeq 0_{d \times d}$ almost surely from Lemma \ref{lem:matrix}. Again, since $\bUtl \Lambda_\tau (\bUtl)^\top \succeq 0_{d \times d}$ almost surely, from Lemma \ref{lem:matrix} we have $\tr \left[\frac{1}{2} \bUlast^2 \Gamma_\tau \Lambda_\tau \bUtl \Lambda_\tau (\bUtl)^\top\right] \geq 0$ almost surely. Therefore, we have
    \begin{equation*}
        \ell_{\rdm}  (\bUtl, \bUlast)\geq - \E \tr \left[\bUlast \Lambda_\tau^2 (\bUtl)^\top\right] 
        = - \tr \left[\bUlast \left(\E \Lambda_\tau^2\right) (\bUtl)^\top\right].
    \end{equation*}
    From Von Neumann's trace inequality (Lemma \ref{lem:trace}) and the fact that $\bUlast(t) > 0$ for any $t\geq 0$ (Lemma \ref{lem:non_zero_random}), we know $\ell_{\rdm} (\bUtl(t), \bUlast(t)) \geq - \sqrt{d} \bUlast \left\|\E \Lambda_\tau^2\right\|_{op} \left\|\bUtl\right\|_F.$
    From Lemma \ref{lem:balanced_random}, we know $\bUlast^2 = \tr(\bUtl (\bUtl)^\top) = \left\|\bUtl\right\|_F^2.$ Since $\bUlast(t) > 0$ for any time, we know actually $\bUlast(t) = \left\|\bUtl(t)\right\|_F$. So we have
    \begin{equation*}
        \ell_{\rdm} (\bUtl(t), \bUlast(t)) \geq - \sqrt{d} \bUlast(t)^2 \left\|\E \Lambda_\tau^2\right\|_{op}.
    \end{equation*}
    From the proof of Lemma \ref{lem:non_zero_random}, we know
    \begin{equation*}
        \ell_{\rdm}(\bUtl(t), \bUlast(t)) \leq \ell_{\rdm}( \bUtl(0), \bUlast(0))
        \leq \frac{\sigma^4 \sqrt{d}}{2} \left[\E \left\|\Gamma_\tau\right\|_{op} \left\|\Lambda_\tau\right\|_F^2 \right] - \sigma^2 \left\|\E \Lambda_\tau \Theta\right\|_F^2.
    \end{equation*}
    Combine the two preceding displays above, we have
    \begin{equation*}
        \bUlast(t) \geq \sqrt{\frac{\sigma^2}{2\sqrt{d}\left\|\E \Lambda_\tau^2\right\|_{op}} \left[2\left\|\E \Lambda_\tau \Theta\right\|_F^2 - \sqrt{d} \sigma^2 \left[\E \left\|\Gamma_\tau\right\|_{op} \left\|\Lambda_\tau\right\|_F^2 \right]\right]} > 0.
    \end{equation*}
    The last inequality comes from Lemma \ref{lem:non_zero_random}.
\end{proof}

\ 

Finally, we prove the PL Inequality, which naturally leads to the global convergence.

\begin{restatable}{lemma}{plinequalityrandom}\label{lemma:proxypl_random}
We do gradient flow on $\ell_\rdm$ with respect to $u_{i,j} \ (\forall i,j \in [d])$ and $\bUlast$. Suppose the initialization satisfies Assumption \ref{assume_init} and the initial scale satisfies \eqref{eqn:init_cond_random_proof}.
If we denote
\begin{equation*}
    \eta = \min\left\{\gamma_i, i \in [d]; \zeta_{ij}, i \neq j \in [d]\right\}
\end{equation*}
and
\begin{equation}\label{eqn:def_nu_maintext}
    \nu := \frac{\eta \cdot \sigma^2}{2\sqrt{d}\left\|\E \Lambda_\tau^2\right\|_{op}} \left[2\left\|\E \Lambda_\tau \Theta\right\|_F^2 - \sqrt{d} \sigma^2 \left[\E \left\|\Gamma_\tau\right\|_{op} \left\|\Lambda_\tau\right\|_F^2 \right]\right] > 0,
\end{equation}
then for any $t \geq 0,$ it holds that
\begin{equation}\label{eqn:PL_random}
    \left\|\nabla \ell_\rdm (\bUtl, \bUlast)\right\|_2^2 := \sum_{i,j=1}^d\left|\frac{\partial \ell_\rdm}{\partial u_{ij}}\right|^2 + \left|\frac{\partial \ell_\rdm}{\partial \bUlast}\right|^2 
    \geq \nu \left(\ell_\rdm - \min \ell_\rdm\right).
\end{equation}
Additionally, $\ell_\rdm$ converges to the global minimal value, $u_{ij}$ and $\bUlast$ converge to the following limits,
\begin{equation}
    \lim_{t \to \infty} u_{ij}(t) = \mathbb{I}(i=j) \cdot \left[\sum_{i=1}^d \frac{\xi_i^2}{\gamma_i^2}\right]^{-\frac{1}{4}} \cdot \frac{\xi_i}{\gamma_i}\quad \forall i \in [d], \quad
    \lim_{t \to \infty} \bUlast(t) = \left[\sum_{i=1}^d \frac{\xi_i}{\gamma_i}\right]^{\frac{1}{4}}.
\end{equation}
\end{restatable}

Translating back to the original parameterization, we have this is equivalent to
\begin{align*}
    \lim_{t \to \infty} \WKQ(t)
    &= \begin{pmatrix}[1.5]
        \left\|\left[\E \Gamma_\tau \Lambda_\tau^2\right]^{-1} \E \left[\Lambda_\tau^2\right]\right\|_F^{-\frac{1}{2}} \cdot \left[\E \Gamma_\tau \Lambda_\tau^2\right]^{-1} \E \left[\Lambda_\tau^2\right] & 0_d \\
        0_d^\top & 0
    \end{pmatrix}, \\
    \lim_{t \to \infty} \WPV(t) 
    &= \begin{pmatrix}[1.5]
    0_{d \times d} & 0_d \\
    0_d^\top & \left\|\left[\E \Gamma_\tau \Lambda_\tau^2\right]^{-1} \E \left[\Lambda_\tau^2\right]\right\|_F^{\frac{1}{2}}
    \end{pmatrix},
\end{align*}
where $\Gamma_\tau = \frac{N+1}{N} \Lambda_\tau + \frac{1}{N}\tr(\Lambda_\tau)I_d \in \R^{d \times d}$ and $\E$ is over $\Lambda_\tau.$

\begin{proof}
First, we prove the PL Inequality. From Lemma \ref{lem:min_loss_random}, we know
\begin{equation*}
    \ell_{\rdm} (\bUtl, \bUlast) - \min \ell_{\rdm} = \frac{1}{2} \sum_{i=1}^d \gamma_i \left(u_{ii} \bUlast - \frac{\xi_i}{\gamma_i}\right)^2 + \frac{1}{2} \sum_{i\neq j} \zeta_{ij} \bUlast^2 u_{ij}^2,
\end{equation*}
where $\xi_i, \zeta_{ij}, \gamma_i$ are defined in \eqref{eqn:def_gamma_zeta_xi}. Meanwhile, we calculate the square norm of the gradient of $\ell_\rdm$:
\begin{align*}
    \left\|\nabla \ell_\rdm (\bUtl, \bUlast)\right\|_2^2 
    &:= \sum_{i,j=1}^d\left|\frac{\partial \ell_\rdm}{\partial u_{ij}}\right|^2 + \left|\frac{\partial \ell_\rdm}{\partial \bUlast}\right|^2 
    \geq \sum_{i,j=1}^d\left|\frac{\partial \ell_\rdm}{\partial u_{ij}}\right|^2 \\
    &= \sum_{i=1}^d \gamma_i^2 \bUlast^2 \left(u_{ii} \bUlast - \frac{\xi_i}{\gamma_i}\right)^2 + \sum_{i\neq j} \zeta_{ij}^2 \bUlast^4 u_{ij}^2.
\end{align*}
Comparing the two displays above, we know in order to achieve $ \left\|\nabla \ell_\rdm\right\|_2^2 \geq \nu \left(\ell_{\rdm} - \min \ell_{\rdm}\right),$ it suffices to make
\begin{align*}
    \gamma_i \bUlast(t)^2 &\geq \frac{\nu}{2} \quad \forall i \in [d], \\
    \zeta_{ij} \bUlast(t)^2 & \geq \frac{\nu}{2} \quad \forall i \neq j \in [d].
\end{align*}
We define $\eta := \min \left\{\gamma_i, \zeta_{ij}, i \neq j \in [d]\right\},$ then it is sufficient to make
\begin{equation*}
    \eta \bUlast(t)^2 \geq \frac{\nu}{2}.
\end{equation*}
From Lemma \ref{lem:lower_bound_u_random}, we know that we can actually lower bound $\bUlast$ from below by a positive constant. Then, the inequality holds if we take
\begin{equation*}
    \nu := \frac{\eta \cdot \sigma^2}{2\sqrt{d}\left\|\E \Lambda_\tau^2\right\|_{op}} \left[2\left\|\E \Lambda_\tau \Theta\right\|_F^2 - \sqrt{d} \sigma^2 \left[\E \left\|\Gamma_\tau\right\|_{op} \left\|\Lambda_\tau\right\|_F^2 \right]\right] > 0.
\end{equation*}
Therefore, as long as we take $\nu$ as above, a PL inequality holds for $\ell_\rdm.$

With an abuse of notation, let us write $\ell_\rdm(t) =  \ell_\rdm(\bUtl(t), \bUlast(t))$.   Then, from the dynamics of gradient flow and the PL Inequality (\eqref{eqn:PL_random}), we know
\begin{align*}
    \frac{\mathrm{d}}{\mathrm{d} t} \left[\ell_\rdm (t) - \min\ell_\rdm\right]
    = - \left\|\nabla \ell_\rdm (t)\right\|_2^2 
    \leq - \nu \left(\ell_\rdm (t)- \min \ell_\rdm\right),
\end{align*}
which by Gr\"onwall's inequality implies 
\begin{equation*}
    0
    \leq \ell_\rdm(t) - \min \ell_\rdm 
    \leq \exp(-\nu t) \left[\ell_\rdm(0) - \min\ell_\rdm\right] \to 0
\end{equation*}
when $t \to \infty.$ From Lemma \ref{lem:min_loss_random}, we know 
\begin{equation*}
    \sum_{i=1}^d \gamma_i \left(u_{ii} \bUlast - \frac{\xi_i}{\gamma_i}\right)^2 + \sum_{i\neq j} \zeta_{ij} \bUlast^2 u_{ij}^2 \to 0 \text{ when } t \to \infty.
\end{equation*}
This implies
\begin{equation}\label{eqn:convergence_uij_random}
\begin{aligned}
    u_{ii} \bUlast &\to \frac{\xi_i}{\gamma_i} \quad \forall i \in [d], \\
    u_{ij} \bUlast &\to 0 \quad \forall i \neq j \in [d].
\end{aligned}
\end{equation}
We take square of $u_{ii}(t) \bUlast(t)$ and $u_{ij}(t) \bUlast(t)$, then sum over all $i,j \in [d].$ Then, we get $\bUlast^2 \sum_{i,j = 1}^d u_{ij}^2 \to \sum_{i=1}^d \frac{\xi_i^2}{\gamma_i^2}.$ From Lemma \ref{lem:balanced_random}, we know for any $t \geq 0,$ 
$\bUlast(t)^2 = \tr\left(\bUtl (\bUtl)^\top\right) = \sum_{i,j = 1}^d u_{ij}^2.$
So we have
\begin{equation*}
    \bUlast(t)^4 = \bUlast^2 \sum_{i,j = 1}^d u_{ij}^2 \to \sum_{i=1}^d \frac{\xi_i^2}{\gamma_i^2},
\end{equation*}
which implies
\begin{equation}\label{eqn:convergence_ulast_random}
    \bUlast(t) \to \left[\sum_{i=1}^d \frac{\xi_i^2}{\gamma_i^2}\right]^{\frac{1}{4}}
\end{equation}
when $t \to \infty.$ Combining \eqref{eqn:convergence_uij_random} and \eqref{eqn:convergence_ulast_random}, we conclude
\begin{equation*}
    u_{ij}(t) \to 0 \quad \forall i \neq j \in [d], \quad
    u_{ii}(t) \to \left[\sum_{i=1}^d \frac{\xi_i^2}{\gamma_i^2}\right]^{-\frac{1}{4}} \cdot \frac{\xi_i}{\gamma_i}\quad \forall i \in [d].
\end{equation*}
\end{proof}

\section{Technical lemmas}
\begin{lemma}[Matrix Derivatives, Kronecker Product and Vectorization, \citep{petersen2008matrix}]\label{lem:matrix}
    We denote $\boldsymbol{A}, \boldsymbol{B}, \boldsymbol{X}$ as matrices and $\boldsymbol{x}$ as vectors. Then, we have
    \begin{itemize}
        \item $\frac{\partial \mathbf{x}^\top \mathbf{B} \mathbf{x}}{\partial \mathbf{x}}=\left(\mathbf{B}+\mathbf{B}^\top\right) \mathbf{x}.$
        \item $\vector(\mathbf{A X B})=\left(\mathbf{B}^\top \otimes \mathbf{A}\right) \vector(\mathbf{X})$.
        \item $\operatorname{tr}\left(\mathbf{A}^\top \mathbf{B}\right)=\vector(\mathbf{A})^\top \vector(\mathbf{B}).$
        \item $\frac{\partial}{\partial \mathbf{X}} \operatorname{tr}\left(\mathbf{X B X}^\top\right)=\mathbf{X B}^\top + \mathbf{X B}.$
        \item $\frac{\partial}{\partial \mathbf{X}} \operatorname{tr}\left(\mathbf{A} \mathbf{X}^\top\right)=\mathbf{A}.$
        \item $\frac{\partial}{\partial \mathbf{X}} \operatorname{tr}\left(\mathbf{A X B X ^ { \top }} \mathbf{C}\right)=\mathbf{A}^\top \mathbf{C}^\top \mathbf{X} \mathbf{B}^\top+\mathbf{C A X B}.$
    \end{itemize}
\end{lemma}

\ 

\begin{lemma}\label{lem:fourth_moment}
    If $X$ is Gaussian random vector of $d$ dimension, mean zero and covariance matrix $\Lambda,$ and $A \in \mathbb{R}^{d \times d}$ is a fixed matrix. Then
    \begin{equation*}
        \E \left[XX^\top A XX^\top\right] = \Lambda \left(A + A^\top\right) \Lambda + \operatorname{tr}(A \Lambda) \Lambda.
    \end{equation*}
\end{lemma}

\begin{proof}
    We denote $X = (X_1,...,X_d)^\top.$ Then,
    \begin{equation*}
        XX^\top A XX^\top = X (X^\top A  X) X^\top = \left(\sum_{i,j=1}^d A_{ij}X_i X_j\right) XX^\top.
    \end{equation*}
    So we know $(XX^\top A XX^\top)_{k,l} = \left(\sum_{i,j=1}^d A_{ij} X_i X_j\right) X_k X_l.$ From Isserlis' Theorem in probability theory (Theorem 1.1 in \citet{michalowicz2009isserlis}, originally proposed in \citet{wick1950evaluation}), we know for any $i,j,k,l \in [d],$ it holds that
    \begin{equation*}
        \E \big[X_i X_j X_k X_l\big] = \Lambda_{ij} \Lambda_{kl} + \Lambda_{ik} \Lambda_{jl} + \Lambda_{il} \Lambda_{jk}.
    \end{equation*}
    Then, we have for any fixed $k,l \in [d],$ 
    \begin{align*}
        \E (XX^\top A XX^\top)_{k,l} 
        &= \sum_{i,j = 1}^d A_{ij} \Lambda_{ij} \Lambda_{kl} + A_{ij} \Lambda_{ik} \Lambda_{jl} + A_{ij} \Lambda_{il} \Lambda_{jk} \\
        &= \operatorname{tr}(A\Lambda) \Lambda_{kl} + \Lambda_k^\top (A + A^\top) \Lambda_l.
    \end{align*}
    Therefore, we know
    \begin{equation*}
        \E (XX^\top A XX^\top) = \Lambda \left(A + A^\top\right) \Lambda + \operatorname{tr}(A \Lambda) \Lambda.
    \end{equation*}
\end{proof}

\

\begin{lemma}[Von-Neumann's Trace Inequality]\label{lem:trace}
    Let $U, V \in \mathbb{R}^{d \times n}$ with $d \leq n$. We have
$$
\operatorname{tr}\left(U^{\top} V\right) \leq \sum_{i=1}^d \sigma_i(U) \sigma_i(V) \leq\|U\|_{\mathrm{op}} \times \sum_{i=1}^d \sigma_i(V) \leq \sqrt{d} \cdot\|U\|_{\mathrm{op}}\|V\|_F
$$
where $\sigma_1(X) \geq \sigma_2(X) \geq \cdots \geq \sigma_d(X)$ are the ordered singular values of $X \in \mathbb{R}^{d \times n}$.
\end{lemma}

\

\begin{lemma}[\citep{meenakshi1999product}]\label{lem:psd}
    For any two positive semi-definitive matrices $A, B \in \mathbb{R}^{d \times d},$ we have
    \begin{itemize}
        \item $\operatorname{tr}[AB] \geq 0.$
        \item $AB \succeq 0$ if and only if $A$ and $B$ commute.
    \end{itemize}
\end{lemma}

\section{Experiment details}\label{appendix_experiments}
In this section, we provide more details for the experiment in Figure~\ref{figure}.  Our experimental setup is based on the codebase provided by~\citet{garg2022can}, with a modification that allows for the possibility that the covariate distribution changes across prompts.  We use the standard GPT2 architecture with 256 embedding size, 12 layers and 8 heads~\citep{radford2018improving} as implemented by HuggingFace~\citep{wolf2020transformers}.  For the GPT2 models, we use the embedding method proposed by~\citet{garg2022can}, where instead of concatenating $x$ and $y$ into a single token, they are treated as separate tokens.   It is also worth noting that the training objective function for the GPT2 model is different than those we consider for the linear self-attention network: for the GPT2 model, the objective function is the average over the full length of the context sequence (predictions for each $x_i$ using $(x_k, y_k)_{k<i}$), while in our setting the objective function is only for the final query point.  However, in the figure, for both GPT2 and the linear self-attention model the error plotted corresponds to the error for predicting the final query point.

In all experiments, covariates are sampled from a mean-zero Gaussian in $d=20$ dimensions with either fixed or random covariance matrix.   
 For the fixed covariance case, we fix the covariance matrix to be identity; for the random case, the covariance matrices are restricted to be diagonal and all diagonal entries are i.i.d. sampled from the standard exponential distribution. The linear weights in all tasks are i.i.d. sampled from standard Gaussian distribution and also independently from all covariates. 
We trained the model for $500000$ steps using Adam~\citep{kingma2014adam} with a batch size of 64 and learning rate of $0.0001.$ We use the same curriculum strategy of~\citet{garg2022can} for acceleration.

For testing the trained model, we used ordinary least squares as a baseline which is optimal for noiseless linear regression tasks.  For prompts at test time, covariates are sampled i.i.d. from a mean-zero Gaussian distribution.  For the fixed-covariance evaluation, the covariance is the identity matrix.  In the random-covariance evaluation, the covariance is a random diagonal matrix with diagonal entries sampled from the standard exponential distribution, multiplied by a scaling coefficient $c \in \{1,4,9\}$, i.e. for each task $\tau,$ the covariance matrix in the random case is 
\begin{equation*}
    \Lambda_\tau = c \cdot \operatorname{diag}\left(\lambda_{\tau,1},...,\lambda_{\tau,d}\right)
\end{equation*}
where $\lambda_{\tau,i} \iid \mathsf{Exponential}(1)$ for any $\tau$ and $i \in [d].$   The plots in Figure~\ref{figure} show the error averaged over $64^2$ prompts, where we sample $64$ covariance matrices for each curve and $64$ prompts for each covariance matrix. We compute $90\%$ confidence interval over 1000 bootstrap trials for each teat.

\printbibliography
\end{document}